\documentclass{article}
\usepackage{graphicx}
\usepackage{pdfpages}
\usepackage{hyperref}
\usepackage{times}
\usepackage{xcolor}
\usepackage{amsmath,amssymb}
\usepackage{bbm}
\usepackage{enumitem}
\usepackage{booktabs}  

\usepackage{microtype}
\usepackage{graphicx}
\usepackage{subcaption}
\usepackage{booktabs} 

\usepackage{hyperref}

\usepackage[preprint]{icml2026}

\usepackage{amsmath}
\usepackage{amssymb}
\usepackage{mathtools}
\usepackage{amsthm}

\usepackage{cleveref}

\theoremstyle{plain}
\newtheorem{theorem}{Theorem}[section]
\newtheorem{proposition}[theorem]{Proposition}
\newtheorem{lemma}[theorem]{Lemma}

\theoremstyle{definition}

\theoremstyle{remark}
\newtheorem{remark}[theorem]{Remark}

\usepackage[textsize=tiny]{todonotes}

\icmltitlerunning{Regularized policy gradient for MAB}

\newcommand{\R}{\mathbb R}

\newcommand{\E}{\mathbb E}

\newcommand{\one}{\mathbbm{1}}

\newcommand{\Pibar}{\bar{\Pi}}

\newcommand{\Dcal}{\mathcal{D}}
\newcommand{\Fcal}{\mathcal{F}}
\newcommand{\Gcal}{\mathcal{G}}
\newcommand{\Hcal}{\mathcal{H}}
\newcommand{\Ocal}{\mathcal{O}}
\newcommand{\Pcal}{\mathcal{P}}
\newcommand{\Rcal}{\mathcal{R}}

\newcommand{\Lcal}{\mathcal{L}}

\begin{document}

\twocolumn[
  \icmltitle{Vanishing L2 regularization 
  for the softmax Multi Armed Bandit}



  \icmlsetsymbol{equal}{*}

  \begin{icmlauthorlist}
    \icmlauthor{\c{S}tefana-Lucia Ani\c{t}a}{equal,ceremade}
    \icmlauthor{Gabriel Turinici}{equal,acadiasi}
  \end{icmlauthorlist}

  \icmlaffiliation{ceremade}{CEREMADE, Universit\'e Paris Dauphine - PSL, CNRS, Paris, France}
  \icmlaffiliation{acadiasi}{`Octav Mayer'' Institute of Mathematics of the Romanian Academy, Bd. Carol I 8, Ia\c{s}i 700505, Romania }

  \icmlcorrespondingauthor{Gabriel Turinici}{gabriel.turinici@dauphine.fr}

  \icmlkeywords{Reinforcement Learning, Multi Armed Bandit, 
  policy gradient regularized  multi armed bandit,
  regularized multi armed bandit, 
  L2 regularized multi armed bandit, 
  Stochastic Gradient Descent Algorithm, Policy Gradient, Regularized Policy Gradients, Proximal Policy Optimization}

  \vskip 0.3in
]

\printAffiliationsAndNotice{}  
\begin{abstract}
Multi Armed Bandit (MAB) algorithms are a cornerstone of reinforcement learning and have been studied both theoretically and numerically. One of the most commonly used implementation uses a softmax mapping to prescribe the optimal policy
and served as the foundation for downstream algorithms, including REINFORCE. Distinct from vanilla approaches, we consider here the L2 regularized softmax policy gradient where a quadratic term is subtracted from the mean reward. Previous studies 
exploiting convexity 
failed to identify a suitable theoretical framework to analyze its convergence when the regularization parameter vanishes. We prove here theoretical convergence results and confirm empirically that 
this regime makes the L2 regularization numerically advantageous on standard benchmarks.
\end{abstract}
\section{Introduction}

Backed by efficient practical achievements in diverse domains such as 
video games \cite{mnih_human_level_rl_2015},
game playing (e.g., Go \cite{silver_rl_mastering_go_2016}),  
conversational AI (e.g., ChatGPT \cite{openai2024gpt4}),
 recommender systems \cite{rl_recommender22},
 wifi antenna placement \cite{muMAB_wifi_antenna2018}, 
autonomous driving \cite{bojarski2016end_rl_cars} and
healthcare \cite{RL_medecine1,RL_medecine2},
 Reinforcement Learning (RL) algorithms emerged as promising field of ongoing research. Among the various frameworks within RL, the Multi-Armed Bandit (MAB) model \cite{contextual_mab,slivkins_introduction_mab_2019} stands out for its extensive use in both theoretical analysis and practical applications.

Our attention is directed toward a specific procedure, the 
softmax parameterized policy gradient, cf. \cite{sutton_reinforcement_2018} (Sect. 2.8 \& Chap. 13). But, distinct from this standard setting, we also include $L2$ regularization on the preference vector (see notations in the following section) and investigate both theoretical convergence and numerical performance.

The outline of the work is the following : we recall some 
results from the literature in this section and introduce the main notations in Section \ref{sec:notations}. We state the convergence results in Section \ref{sec:cv_proof}  which are then illustrated numerically in Section~\ref{sec:numerical}. Finally, Section~\ref{sec:conclusion} is devoted to concluding remarks and the Appendix contains all the proofs and further numerical results.

\subsection{Brief literature review}

Policy gradient algorithms have achieved notable success across various reinforcement learning applications and have been adapted to ensure stable convergence. Prominent examples include the log-barrier penalized REINFORCE algorithm \cite{reinforce_cv_proof_Kim_Boyd_2021}, trust-region policy optimization (TRPO) \cite{trpo15}, and proximal policy optimization (PPO)—OpenAI’s standard reinforcement learning framework. All of these methods employ regularization techniques to constrain policy updates and enhance learning stability. In this work, we explore an alternative regularization approach, focusing on its application within the Multi-Armed Bandit (MAB) framework.

Although policy gradient algorithms perform well empirically, their theoretical convergence properties for MABs have only recently been clarified. Fazel et al. \cite{fazel_global_2018} established high-probability convergence results for stochastic gradient methods in linear quadratic regulator problems. Building on this, Agarwal et al. \cite{pmlr_v125_agarwal20a} analyzed convergence within a general Markov process framework and, for the softmax parameterization studied here, considered three approaches: basic policy gradient descent, entropy-regularized policy gradient, and the natural policy gradient method, proving global optimality for the latter.  Still in the context of entropy‑regularized multi‑armed bandits, Ding et al. \cite{ding_entropy_mab24} establish global convergence guarantees for entropy‑regularized softmax policy gradient methods using (nearly) unbiased stochastic estimators, while Aghaei \cite{Aghaei2025SoftmaxPG} studies convergence in bandit and tabular MDP settings and characterizes when convergence holds under linear function approximation.
In contrast, our analysis focuses on the softmax parameterization with $L2$ regularization for which only a recent work  \cite{anita2024convergence} is available when regularization parameter is large enough to ensure convex functional. This precludes the interesting case of vanishing regularization parameter, the only one that can convergence to the optimal solution.

Also recently, Bhandari and Russo \cite{bhandari_global_2024} examined the softmax parameterization under an idealized policy gradient scheme assuming exact gradient evaluations. Here, we instead investigate the more practical case of stochastic (approximate) gradients, under slightly stronger assumptions.

Additional related studies include \cite{wang2019neural}, which explores deep neural network parameterizations, and \cite{zhang_global_2020}, which proposes a Monte Carlo variant with random roll-out horizons for infinite-horizon discounted problems.

In a series of recent works, Mei et al. \cite{mei_global_2020,mei2023stochastic} further advanced this line of research by showing that exact-gradient policy gradient methods with softmax parameterization converge at a rate of 
$O(1/t)$, and that entropy regularization can accelerate convergence. They also show that, at odds with common lore, the stochastic gradient can converge for MAB even for a constant learning rate albeit at the price of having an important number of iterations (see also the comments in \cite{baudry2025does}). 
However, their analysis relies critically on the assumption that the reward distribution is bounded—an assumption we do not make in our work. Furthermore, their framework does not account for the presence of L2 regularization, which is a key component in our setting.

From a broader theoretical perspective, our analysis focuses on softmax-pa\-ra\-me\-te\-rized policy gradients with 
$L2$ regularization. We build on  arguments similar to those used in establishing the convergence of general stochastic gradient descent (SGD), starting with the classical work of Robbins and Monro \cite{robbins_stochastic_1951}. A comprehensive reference on this topic is provided in \cite{chen_stochastic_2002}, see also  \cite{sgd_conv_non_cx20,mertikopoulos_almost_2020}
for recent contributions including non-convex objectives.

Classical results on stochastic approximation also provide a rigorous foundation for the analysis of stochastic gradient descent (SGD) algorithms with time-varying perturbations. Early work by Blum~\cite{blum1954approximation}, and later extensions by Borkar et. al.~\cite{borkar2008stochastic}(Chap. 2), \cite{borkar_ode_2025} and Kushner and Yin~\cite{kushner2003stochastic}(Chap. 5), establish almost sure convergence of iterative methods subject to diminishing step sizes and vanishing regularization terms. 
Further results are available when the gradient is perturbed by a bounded term, see \cite{ajalloeian2021convergencesgdbiasedgradients}; in our case 
the perturbation is the gradient of the L2 regularization term which is not a priori bounded. 
Moreover the regularization component follows a step-size schedule that does not necessarily fit within this framework, see \Cref{rem:gamma_rho_hypothesis} for details.

More precise analyses are available when the criterion is convex, see 
\cite{regularizedSGD_convex} or when the regularization is 
entropic \cite{weissmann2025almost}. 
These studies show that when the regularization coefficient decreases sufficiently fast, the asymptotic behavior of the algorithm coincides with that of the unregularized case. From a practical point of view, introducing a time-dependent quadratic penalty can stabilize the early stages of optimization while maintaining asymptotic unbiasedness. More recent analyses, such as the work of Bach~\cite{bach2015adaptivity}, further emphasize the connection between decaying regularization, stability, and the adaptive properties of averaged SGD in locally strongly convex settings. But applying these general results to Multi Armed Bandit requires some assumptions which are not always satisfied in our setting.

Indeed, convergence results for SGD, such as those in \cite{chen_stochastic_2002} (Thms. 1.2.1 and 1.3.1), rely on several hypotheses, e.g. the uniqueness of the critical point (which does not hold here), boundedness conditions (the optimal 
preference vector is unbounded when regularization is not present) or the existence of a suitable Lyapunov function (the standard one degenerates in this setting) and boundedness of trajectories \cite{mertikopoulos_almost_2020}.
Despite these challenges, our approach still builds upon this theoretical foundation, combining estimates and results from prior literature that, to our knowledge, have not yet been applied in this specific context. 

We are therefore able to obtain first, state of the art, results for the convergence of the L2 regularized softmax policy gradient with vanishing regularization parameter.

\section{The $L2$ regularized softmax parameterized policy gradient Multi Armed Bandit} \label{sec:notations}

The Multi Armed Bandit problem in the formulation of \cite{sutton_reinforcement_2018}, involves a choice among $k$ 
alternatives called 'arms` and indexed by 
\begin{align}
a \in [k]:=\{ 1, 2, \ldots, k\}.    
\end{align}
When chosen, each arm \(a\) gives rewards sampled from some reward distribution $R(a)$; we introduce the notation for the average of the distribution corresponding to arm $a$:
\begin{equation}
	q_*(a) := \E [ R(a) ], \forall a \in [k].
	\label{eq:hyp_mean0}
\end{equation}
For instance, reward can be sampled from a normal variable with mean  
$q_*(a)$ and variance $\sigma(a)^2=1$, but we allow for more general choices (see \Cref{sec:numerical_student}).
The choice selected at time step \(t\) is denoted $A_t$ and the reward 
$R_t \sim R(A_t)$ 
is sampled from the distribution $R(A_t)$ associated with arm $A_t$.
The objective is to maximize the average reward.

We will denote 
\begin{align}
    \Pcal_k =\{\Pi \in (\R_+)^k : \sum_{a \in [k]} \Pi_a=1\},
\end{align}
the set of all probability distributions with outcomes in $[k]$.
\\ 
The softmax policy gradient algorithm works with a 
'preference vector' $H \in \R^k$ that, 
through the softmax mapping~:
\begin{equation}\Pi _H(a)={\frac{e^{H(a)}}{\sum_{b \in [k]}e^{H(b)}}}, 
	\label{eq:definition_piha}
\end{equation}
creates a distribution $\Pi_{H} \in \Pcal_k$ from which the choice is made. 
Here, when the preference vector is $H$, the probability to choose arm $A$ is 
$\Pi_H(A)$.
To this classical framework we add $L2$ regularization 
so that finally the goal is to find the $H\in \mathbb{R}^k$ solution to~:
\begin{align}
& \text{maximize}_{H \in \R^k} \Lcal_\gamma(H),
\label{eq:problem_mab_as_maximization}
\\ &
\text{with} \footnotemark
\ \ 
\Lcal_\gamma(H) :=\mathbb{E}_{A \sim \Pi_H} \left[ R(A) - \frac{\gamma}{2}\| H\|^2\right].
\label{eq:functional_to_min}
\end{align}
\footnotetext{The notation $V \sim \mu$ means that the random variable $V$ is sampled from the distribution $\mu$.}

\noindent
The parameter $\gamma \ge 0$ is the $L2$ regularization coefficient; in particular 
\begin{equation}
\Lcal_0(H) :=\mathbb{E}_{A \sim \Pi_H} \left[ R(A) \right].
\label{eq:Lzero}
\end{equation}

When the dependence of $\gamma$ is not important we will only write 
$\Lcal(H)$ instead of $\Lcal_\gamma(H)$.
The difference with 
 the classical MAB  \cite{sutton_reinforcement_2018}(Sect 2.8) lies in the presence of the regularization  $\frac{\gamma}{2} \|H\|^2$. 
Assuming possible dependence of $\gamma$ on the index $t$, 
problem \eqref{eq:problem_mab_as_maximization} 
is solved using a stochastic gradient ascent algorithm~:
\begin{align}
&H_{t+1}(a) =H_t(a) 
\nonumber \\&
+ \rho_t  \left[ 
(R_t- \bar{R}_t) (\one_{a=A_t}- \Pi_{H_t}(a))- \gamma_t H_t(a)
\right], a \in [k].
\label{eq:definition_Ht_pg}
\end{align}
Here $\bar{R}_t$ is the mean reward up to time $t$ constructed from 
$R_\tau$, $\tau < t$; the use of  $\bar{R}_t$ is not compulsory but has been showed to improve numerical behavior. Note that $\bar{R}_t$ does not bias the right hand side as the overall average is null.
The 'learning rate` (also called 'time step') $\rho_t$ 
and the regularization coefficient $\gamma_t$ are both positive.

\section{Theoretical convergence results when $\gamma_t  \searrow 0$} \label{sec:cv_proof}

We are interested in what happens when $\gamma$ is varying in time; the most natural behavior is to take $\gamma_t$ decreasing in order to gradually diminish the impact of 
regularization and to find a solution of the original problem of maximizing $\Lcal_0$. 

We first recall 
known results
that allow
to see update formula \eqref{eq:definition_Ht_pg} as a stochastic gradient ascent algorithm \`a la Robins and Monro \cite{robbins_stochastic_1951}. 
The reward $R_t$ represents the consequence of choosing arm $A_t$ and  this choice is independent of $\Fcal_t$ but depends on $H_t$ through $\Pi_{H_t}$. Chronologically, the choice $A_t$ is made between time $t$ and $t+1$. Denote  $\Fcal_t$ the filtration constructed with information available up to time step $t$. It is known that (see \cite{sutton_reinforcement_2018})~:
\begin{align}
&	\E [ R_t \one_{a=A_t}| \Fcal_t ] = q_*(a)\Pi_{H_t}(a), \forall a \in [k],
	\label{eq:hyp_mean} \\
&	\E[\one_{a=A_t}|\Fcal_t] = \Pi_{H_t}(a),
	\label{eq:filtration_ft_at}   \\
&	\E [ R_t | \Fcal_t ] = \sum_a \Pi_{H_t}(a) q_*(a) = \Lcal_0(H_t).\label{eq:cond_mean_total_reward} 
\end{align}

As is standard 
we assume from now on that~:\footnote{This will be relaxed in \Cref{sec:numerical_student}.}
\begin{align}
&	\text{ there exists a constant } C_m>0 \text{ such that~: }
\nonumber\\ 
&	\E [ R(a)^2 ] \le C_m, \forall a \in [k].
	\label{eq:hyp_second_moment}
\end{align}
\noindent 
We also assume that $\gamma_t$ is a deterministic function of the time $t$ (all that follows also works for non-deterministic choices independent of $\Fcal_t$ under mild additional assumptions). 
We use the notations (see \Cref{eq:definition_Ht_pg})~:
\begin{align}
&	u_t(a):=(R_t- \bar{R}_t) (\one_{a=A_t}- \Pi_{H_t}(a)),
	\label{eq:def_non_biased_gradient0}
    \\ 
&	g_t(a):=	u_t(a)- \gamma_t H_t(a).
	\label{eq:def_gt}
\end{align}
With these notations, we recall that~\cite{sutton_reinforcement_2018}:
\begin{equation}
	\E    \left[ \left. u_t
	\right|\Fcal_t \right] =   \left. \nabla_H \Lcal_0(H)\right|_{H=H_t},
	\label{eq:non_biased_gradient_ut}
\end{equation}
and thus:
\begin{equation}
\E    \left[ \left. g_t
\right|\Fcal_t \right] =   \left. \nabla_H \Lcal_{\gamma_t}(H)\right|_{H=H_t},
\label{eq:non_biased_gradient_gt}
\end{equation}
which means that the update \eqref{eq:definition_Ht_pg} is indeed an unbiased estimation of the true gradient $\nabla_H \Lcal_{\gamma_t}(H)$.
Note moreover that:\footnote{Here $\odot$ denotes the Hadamard, i.e., componentwise, product of vectors.}
\begin{align}
& 
\Lcal_0(H) = \langle q_*, \Pi_H \rangle,  
\\ &
\nabla_H \Lcal_0(H) = \left( \sum_{b\in [k]} q_*(b) \Pi_H(b) (  \one_{a=b}- \Pi_H(a)) \right)_{a=1}^k 
\nonumber\\ &
= q_* \odot \Pi_H - \Pi_H \cdot \E_{\Pi_H}[q_*]
= \Pi_H \odot \left( q_*- \langle \Pi_H, q_*\rangle 
\right).
    \label{eq:nablaL0}
\end{align}
In addition it was also proved
(see \cite{mei_global_2020,mei2023stochastic})
that for some constant 
$c_{q_*}$ depending only on $q_*$ and $C_m$ and $C_{q_*}= 2c_{q_*}^2$~:
\begin{equation}
	\E [ \|u_t \|^2 ] \le c_{q_*}^2,
	\label{eq:bounded_gradu}
\end{equation}
which shows that:
\begin{equation}
\E [ \|g_t \|^2 |\Fcal_t] \le C_{q_*} + 2 \gamma_t^2 \|H_t\|^2.
\label{eq:bounded_grad}
\end{equation}
\label{lemma:unbiased_grad}

Some convergence results have been proved in the literature but none allows to have insight in the 
 behavior of $H_t$ when the regularization parameter $\gamma_t$ vanishes. This is the main object of this contribution.  
We will assume in the following that~:
\begin{align}
\textbf{Hyp}_\gamma: \ \ \ 
	\gamma_t \text{ is a decreasing sequence, } \ \lim_{t\to\infty} \gamma_t =  0.
	\label{eq:hyp_gamma_decreasing}
\end{align}

\subsection{Non constant $\rho_t$}
When $\rho_t$ depends on $t$  we make the standard assumption:
\begin{align}
& 
\textbf{Hyp}_\rho: \ \ \ 
\sum_{t\ge 0} \rho_t = \infty \text{ and } \sum_{t\ge 0} \rho_t^2 < \infty .
	\label{eq:hyp_rhot}
\end{align}
We start with a technical remark:
\begin{lemma}
Under assumptions~\eqref{eq:hyp_second_moment}, \eqref{eq:hyp_rhot}, \eqref{eq:hyp_gamma_decreasing} denote $t_0$ the first index such that for all $t\ge t_0$: $\rho_t \gamma_t \le 1$\footnote{Such a $t_0$ exists because $\rho_t \gamma_t$ is deterministic and $\rho_t \gamma_t \to 0$.}. Then:
\begin{align}
& 
\forall t \ge t_0: 
	\E[ \gamma_t^2 \|H_t\|^2] \le \max\{ \gamma_{t_0}^2 \E[\|H_{t_0} \|^2], c_{q_*}^2 \}.
\label{eq:boundHzero}
\end{align}
Moreover there exists a constant $C_g >0$ such that~:
\begin{equation}
\E [ \|g_t \|^2 ] \le C_g, \  \forall t \geq 0.
\label{eq:bounded_grad_cg}
\end{equation}
\label{lemma:bounded_gH_gt}
\end{lemma}
\vspace*{-0.5in}
\begin{proof}
See Appendix.
\end{proof}
\begin{remark}
It can be noted that Lemma~\ref{lemma:bounded_gH_gt} remains true when instead of \eqref{eq:hyp_rhot} we only assume $\rho_t \to 0$.
For completeness, the analogous result when $\lim_{t\to\infty} \gamma_t >0$  is given in Appendix, Lemma~\ref{lemma:boundedness_gammabarnonzero}.
\end{remark}

We state now a first result that gives information on the behavior of the sequence $(H_t)_{t\ge0}$, namely the summability of the gradient.
\begin{proposition}
	Under assumptions~
	\eqref{eq:hyp_second_moment}, 	\eqref{eq:hyp_rhot}, \eqref{eq:hyp_gamma_decreasing} 
and  denoting $\Lambda_T =\sum_{t=0}^{T}\rho_t$~:
 \begin{align}
& 
\sum_{t=0}^{\infty }\rho _t\mathbb{E}[\|\nabla _H{\cal L}_{\gamma _t}(H_t)\|^2]<\infty ,
\label{eq:summability}
\\ &
\sum_{t=0}^{\infty } 
\frac{\gamma_t-\gamma _{t+1}}{2}\mathbb{E}[ \| H_{t+1}\|^2]<\infty ,
\label{eq:summability_dgammaH2}
\\ &
\sum_{t=0}^{\infty } 
\frac{\gamma_t-\gamma _{t+1}}{2}\mathbb{E}[ \| H_{t}\|^2]<\infty ,
\label{eq:summability_dgammaH2t}
\\ &
\lim_{T\to \infty} \frac{1}{\Lambda_T} \sum_{t=0}^{T}\rho _t\mathbb{E}[\|\nabla _H{\cal L}_{\gamma _t}(H_t)\|^2]=0 ,
\label{eq:meancv}
\\ & 
\liminf_{t\to \infty}  \mathbb{E}[\|\nabla _H{\cal L}_{\gamma _t}(H_t)\|^2] =0.
\label{eq:liminf}
 \end{align}
\label{prop:summability}
\end{proposition}
\begin{proof}
See Appendix.
\end{proof}
\noindent 
The convergence of the "liminf" of the gradient norm to zero in (\ref{eq:liminf})
 ensures that the algorithm repeatedly approaches stationary points of the objective function, justifying gradient-based stopping criteria, and showing that meaningful progress is achieved despite the inherent stochastic noise of the updates.

\noindent To obtain  results stronger than those in  Proposition \ref{prop:summability}, we will need two  technical lemmas
\ref{lemma:seriesXnYn}
and 
\ref{lemma:summability_seq_cv}
that are stated and proved in the Appendix. 

We will introduce for any $\Pi \in \Pcal_k$ the regret:
\begin{align}
\Rcal(\Pi) := \max_a q_*(a) - \langle q_*,\Pi\rangle.
\label{eq:def_regret}
\end{align}
Regret is the key metric optimized during training and serves as an indicator of result quality. Lower regret implies better performance, with the ideal scenario being no regret at all.
Denote $\delta_a$ the Dirac mass supported in the arm $a$ i.e. $\delta_a=(0,...,1,...,0)$ with $1$ in the position $a \in [k]$
and:
\begin{align}
& 
    \Dcal = \{ \delta_a, a \in [k] \}, \\
& \forall \pi
\in  \R^k : \ \|\pi- \Dcal \| := \min_{\mu \in \Dcal} \| \pi - \mu\|.
\label{eq:def_dist_Dcal}
\end{align}
We  consider several hypotheses, it will be made clear below when each one is invoked:
\begin{align}
&    {\textbf{Hyp-diff}}: \ \
\forall a \neq b \in [k] : \     q_*(a) \neq q_*(b).
\label{eq:hyp_qa_diff_qb} \\
&    {\textbf{Hyp-3}: } \ \ 
\forall a \in [k] : \E[|R(a)|^3] <\infty .
\label{eq:hyp_third_moment} \\
&   {\textbf{Hyp-}\rho\gamma: } \ \ 
 \exists \ c_{\rho\gamma} > 0 \text{ such that: } \nonumber \\ 
 & \forall t\ge 0: \  c_{\rho\gamma} \rho_t \gamma_t^2 \le \gamma_t - \gamma_{t+1}.  
\label{eq:gamma_hyp2}
\end{align}
%

\begin{remark}
Note that under \eqref{eq:hyp_qa_diff_qb} 
 there exists some unique best arm $a^*$ such that: 
\begin{align}
\forall a \in [k]: \ 
q_*({a^*}) \ge q_*(a).
\label{eq:definition_bestarm_astar}
\end{align}    
\end{remark}

As a notational convenience, from now on, for a sequence of random variables $X_t$: 
 \begin{align}
 & X_t 
\xrightarrow[t\to\infty]{L^2} Y 
& \textrm{ means }
\lim_{t\to \infty}  \mathbb{E}[\|X_t-Y\|^2] =0.
\label{eq:def_limL2}
 \end{align}

We can now give the main result for time-varying $\rho_t$ which shows that, under appropriate hypothesis, we can guarantee the convergence of the gradient, regret and the limit distribution. 
\begin{theorem}
Under assumptions~
\eqref{eq:hyp_second_moment}, 	
\eqref{eq:hyp_rhot},
\eqref{eq:hyp_gamma_decreasing}: 
\begin{enumerate}
\item  If in addition \eqref{eq:hyp_third_moment}
holds then:
 \begin{align}
 & \nabla _H{\cal L}_{\gamma _t}(H_t) 
\xrightarrow[t\to\infty]{L^2} 0 .
\label{eq:limLgammat}
 \end{align}
\item
If instead of \eqref{eq:hyp_third_moment}
we have \eqref{eq:gamma_hyp2}
then \eqref{eq:limLgammat} still holds and moreover~:
\begin{align}
&    \sum_{t=0}^{\infty} \rho_t \mathbb{E}[\|\gamma_t H_t\|^2] < \infty,
    \\ &  
    \sum_{t=0}^{\infty} \rho_t  \mathbb{E}[\|\nabla _H{\cal L}_{0}(H_t)\|^2]< \infty,   
\label{eq:summability_L0gtHt}
\\
&
\text{the sequences }\left( \E[\gamma_t \| H_t\|^2] \right)_{t\ge 0}, 
\left( \E[\Lcal_{0}(H_t)] \right)_{t\ge 0} 
\nonumber \\&
\textrm{ and } 
\left( \E[\Lcal_{\gamma_t}(H_t)] \right)_{t\ge 0} 
\textrm{ converge},
\label{eq:gtHt2_cv}
\\
 &
 \gamma_t H_t \xrightarrow[t\to\infty]{L^2} 0,
\label{eq:lim_gammatHt}
 \\
&
\nabla _H{\cal L}_{0}(H_t) 
\xrightarrow[t\to\infty]{L^2} 0. 
\label{eq:limL0}
 \end{align}
\item 
Assuming~\eqref{eq:hyp_qa_diff_qb}  and
 \eqref{eq:gamma_hyp2} 
hold true then:
\begin{align}
\Pi_{H_t}\xrightarrow[t\to\infty]{L^2} \Dcal. 
\label{eq:lim_in_Dcal}
\end{align}
\label{item:lim_dist_zero}
\item If in addition to 
\eqref{eq:hyp_qa_diff_qb} and
\eqref{eq:gamma_hyp2}
we have\footnote{See \Cref{eq:definition_bestarm_astar} for definition of $a^*$.}:
\begin{align}
 & \inf_{t\geq 1} \Pi_{H_t}(a^*) \ge c_0 > 0, a.s. &
\label{eq:hyp_liminf_piastar_positive}
\\
& \mathrm{then: } &\nonumber \\
& \Rcal(\Pi_{H_t})\xrightarrow[t\to\infty]{L^2}0 ,&
\label{eq:lim_regret_zero}
\\ 
&\Pi_{H_t}\xrightarrow[t\to\infty]{L^2} \delta_{a^*}. &
\label{eq:lim_qstar} 
\end{align}
\end{enumerate}
\label{thm:cv_nabla}
\end{theorem}
\begin{proof}
See Appendix.
\end{proof}

\begin{remark}
A stronger convergence result than  \eqref{eq:lim_in_Dcal}, but under more restrictive hypotheses, is given in Appendix  Lemma~\ref{lemma:cv_pi}.
\end{remark}

\begin{remark}
The result shows that in the limit of large $t$, the distribution 
$\Pi_{H_t}$ will converge to a Dirac mass; under suitable hypothesis this is optimal and the regret is also converging to zero.

Such a result may seem at first non-standard as one may rather expect that we prove that $H_t$, the primal optimization variable, converges to some optimal value. This is not possible because the best reward value is realized when $\Pi_H$ is a sum of Dirac masses supported in the non-empty set 
\begin{align}
\Ocal = \{ a^* \in [k] : q_*(a^*) \ge q_*(\ell), \ \forall \ell \in [k] \}   
\end{align}
(argmax of $q_*$). Such a distribution is only reached when $H(a) = \infty$ for any $a\in \Ocal$ so there is no finite solution that $H_t$ could converge to.
\end{remark}
\begin{remark}
Proving the convergence towards a critical point in \eqref{eq:limL0}
does not require complicated hypotheses. For the convergence of the regret
\eqref{eq:lim_regret_zero}, the assumption \eqref{eq:hyp_qa_diff_qb} is standard in the literature ; the hypothesis  \eqref{eq:hyp_liminf_piastar_positive}
is a strong assumption, but 
(when $\rho_t$ is not constant) 
is still required in state of the art results (even without regularization) unless the support is considered bounded. Of course, a general proof without using \eqref{eq:hyp_liminf_piastar_positive} would be desirable. Note that this is exactly what the $\epsilon$-greedy policies, very popular in reinforcement learning, try to do because, unless the rewards are bounded, there is nothing that prevents some exceptional reward obtained from a non-optimal arm to perturb the optimization indefinitely.
\end{remark}

\begin{remark}
To the best of our knowledge, the assumption \eqref{eq:gamma_hyp2} is never used in the literature; 
this regime is different from the one in \cite{kushner2003stochastic} that are closer to the classical Robbins-Monro setting \cite{robbins_stochastic_1951}; the offending part is the $\gamma_t$ sequence whose sum is not necessarily diverging and $\sum_t \gamma_t^2$ is not necessarily converging; therefore the tools presented there do not apply. \\ 
The assumption \eqref{eq:gamma_hyp2}
is satisfied for the linear decay schedule $\gamma_t = \frac{1}{c_1 + c_2 \cdot t}$ when $\rho_t$ remains bounded (and in particular for any $\rho_t$ satisfying \eqref{eq:hyp_rhot}). When $\rho_t= \frac{1}{c_1' + c_2' \cdot t}$ many functional forms 
ensure \eqref{eq:gamma_hyp2}, 
such as 
$\gamma_t = \frac{1}{c_1'' + c_2'' \cdot t^\alpha}$, $\alpha >0$ and 
$\gamma_t = \frac{1}{c_1''' + c_2''' \cdot \log{t}}$ (note that in this case both
$\sum_t \rho_t \gamma_t$ and
$\sum_t \gamma_t^2$ are diverging). On the contrary, for $\gamma_t = \frac{1}{c_3 + c_3\cdot \log(\log{t})}$
linear schedule for $\rho_t$ is not enough but $\rho_t = 
\frac{1}{c_1' + c_2' \cdot t \log t}$ (which fulfills \eqref{eq:hyp_rhot}) will be.
In general, for a sequence $\rho_t$ there always exists a sequence $\gamma_t$ satisfying  assumption \eqref{eq:gamma_hyp2}, take for instance 
$\gamma_t = \frac{c}{c_0 +\sum_{\tau=1}^{t-1} \rho_\tau}$ where $c>0$ and $c_0\ge 0$ are arbitrary constants. Note that in this case 
$\sum_t \rho_t \gamma_t = \infty$ which is not a regime already seen in the literature.
%
%
\label{rem:gamma_rho_hypothesis}
\end{remark}

\subsection{Constant $\rho_t=\rho$}

We now switch to the situation when $\rho_t$ is a constant, denoted $\rho$; note that we still honor the assumption \eqref{eq:hyp_gamma_decreasing}. Analogous results to \Cref{thm:cv_nabla} can be formulated, we will only give the most complete conclusions and not all intermediary statements.

To ease notations, when a nonnegative sequence $U_t$ depending on $t$ and $\rho$
is such that for all $t\ge 0, \rho > 0$ we have  $U_t \le c \rho$ for some constant $c$ independent of $\rho$ (and $t$) we will denote $U_t = O(\rho)$; same for 
notation $O( \sqrt{\rho})$.
%
\begin{theorem} Let $\rho_t=\rho$ (constant) and denote 
\begin{align}
\bar{\Pi}_t= \frac{\sum_{s=1}^t \Pi_{H_s}}{t}.     
\end{align}
Under assumptions~
\eqref{eq:hyp_second_moment}, \eqref{eq:hyp_gamma_decreasing},   \eqref{eq:gamma_hyp2}
and \eqref{eq:hyp_qa_diff_qb}:
\begin{align}
 &
 \frac{\sum_{s=1}^t
 \E[\| \nabla _H{\cal L}_{\gamma_s}(H_s)\|^2] 
 }{t} =O(\rho),
\label{eq:limLgammat_rhocst} \\
& \Rcal(\Pibar_t) =O( \sqrt{\rho}),
\label{eq:lim_regret_zero_rhocst}
\\ 
& \| \Pibar_t -  \delta_{a^*} \| =O( \sqrt{\rho}).
\label{eq:lim_qstar_rhocst}
\end{align}
\label{thm:cv_nabla_rhocst}
\end{theorem}
\vspace*{-0.35in}
\begin{proof}
See Appendix.
\end{proof}
The interpretation of the result is as follows: when the step $\rho$ is small enough, the procedure is converging with an error of order $\rho$.
Of course, ideally, one would like to obtain stronger results as in \Cref{thm:cv_nabla} while here we do not have information directly on $\Pi_{H_t}$ but rather on its average $\Pibar_t$. However, as we will see in the numerical results, $\Pi_{H_t}$ behaves well in practice.

\section{Numerical simulations} \label{sec:numerical}

\begin{figure}[ht]
  \begin{center}
    \centerline{
    \includegraphics[width=.9\columnwidth]{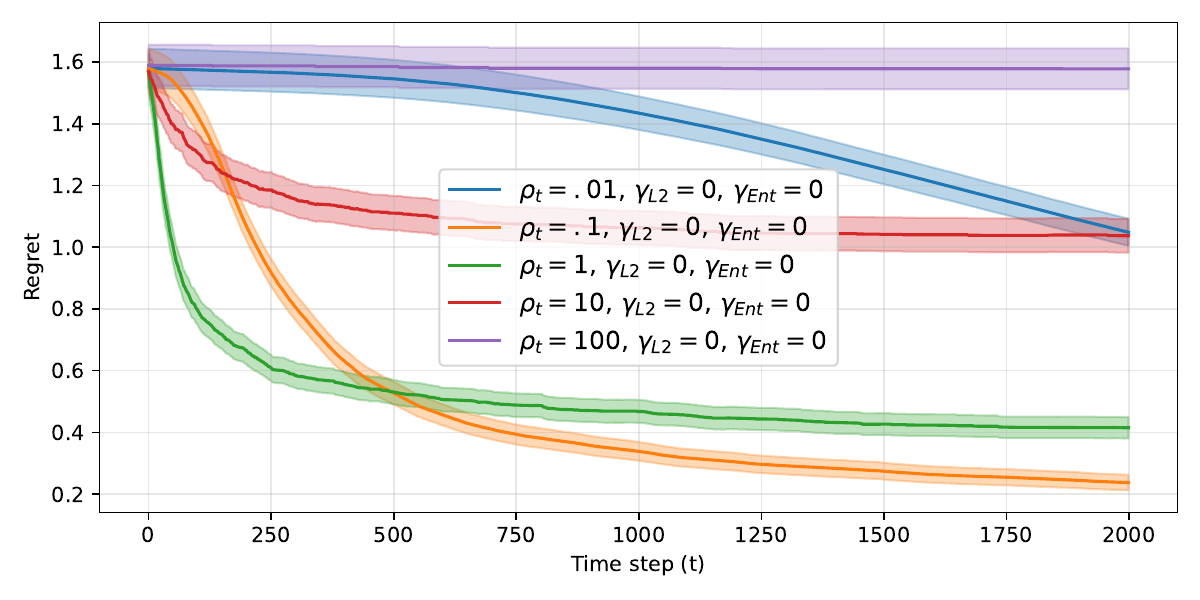}
}%
\caption{The average regret and 95\% confidence intervals when $\rho_t$ is constant, spanning several orders of magnitude from 0.01 to 100. Initial distribution 
$\Pi_{H_0}$ corresponds to 
 $H_0=(5,...,0)$. No regularization (neither L2 nor entropic) is used. The value $\rho_t=0.1$ appears to be the clear winner.
 This is used as baseline for latter comparisons (also coherent with results in  
 \Cref{fig:grid_search_linear_schedule_rhot} in appendix where an extensive grid search for time-dependent $\rho_t$ is presented).}   \label{fig:constant_rho_no_reg}
  \end{center}
\end{figure}

\begin{figure}[ht]
  \begin{center}
    \centerline{
\includegraphics[width=\columnwidth]{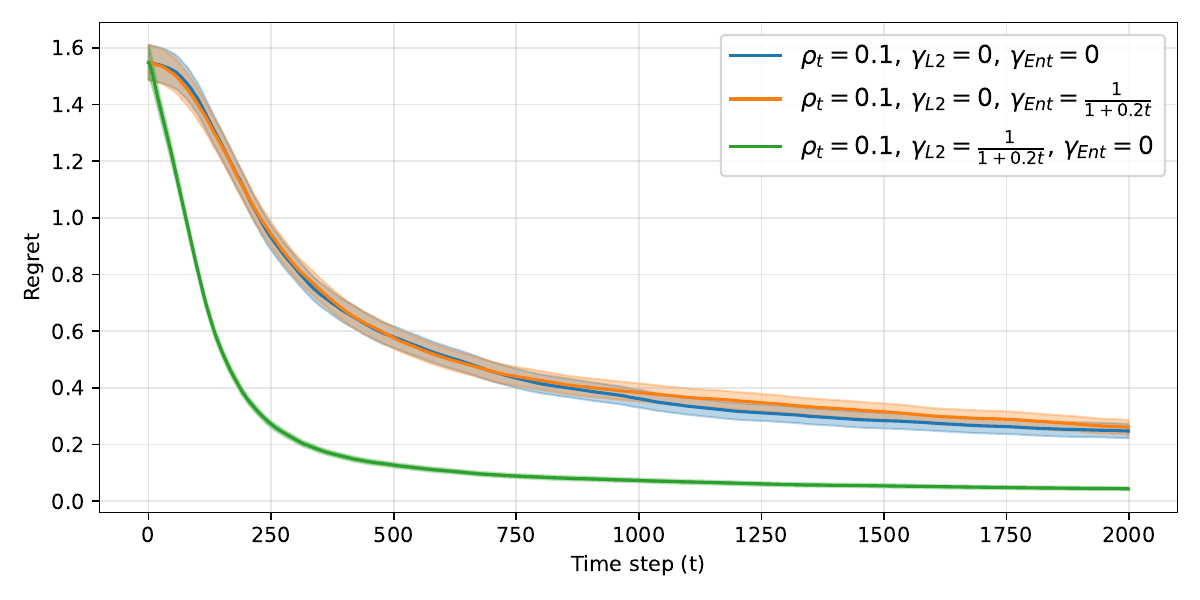}
}
\caption{The average regret and 95\% CI when starting from 
 $H_0=(5,...,0)$, $\rho_t=0.1$ (constant); the non regularized baseline ($\gamma_{L2}=0=\gamma_{Ent}$ is compared with the $L2$ (only) regularization
schedule
$\gamma_{L2}(t) = \frac{\gamma_0}{1 + 0.2 \cdot t}$
$\gamma_{Ent}(t) = 0$
and with the entropy only regularization schedule 
$\gamma_{L2}(t) = 0$
$\gamma_{Ent}(t) = \frac{\gamma_0}{1 + 0.2 \cdot t}$.
The L2 regularization seems to perform best while the entropy does not seem to help much.}  \label{fig:cst_rho_compare_L2_reg_entropy}
  \end{center}
\end{figure}

\begin{figure}[ht]
  \begin{center}
    \centerline{
\includegraphics[width=.95\columnwidth]{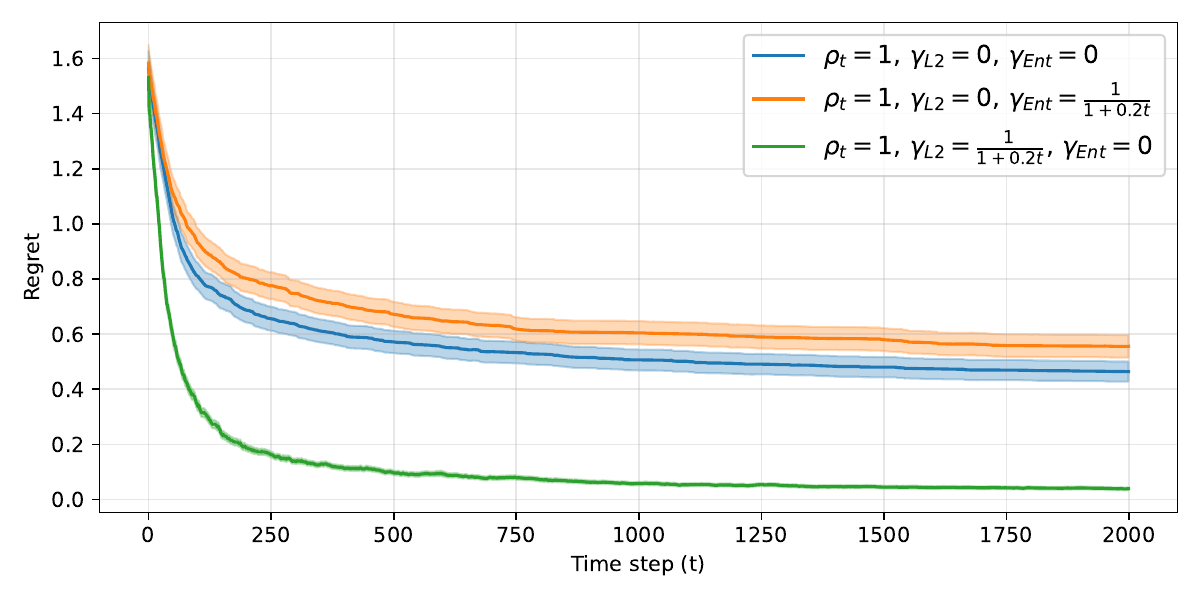}
}
\caption{
Same as in \Cref{fig:cst_rho_compare_L2_reg_entropy} except that for the main step $\rho_t$ we use a non-optimal value (constant); it is seem that the L2 regularization helps 'correct' the bad $\rho_t$ value.} 
\label{fig:cst_rho1_compare_L2_reg_entropy}
  \end{center}
\end{figure}

\begin{figure}[htbp!]
\begin{center}
\includegraphics[width=.9\columnwidth]{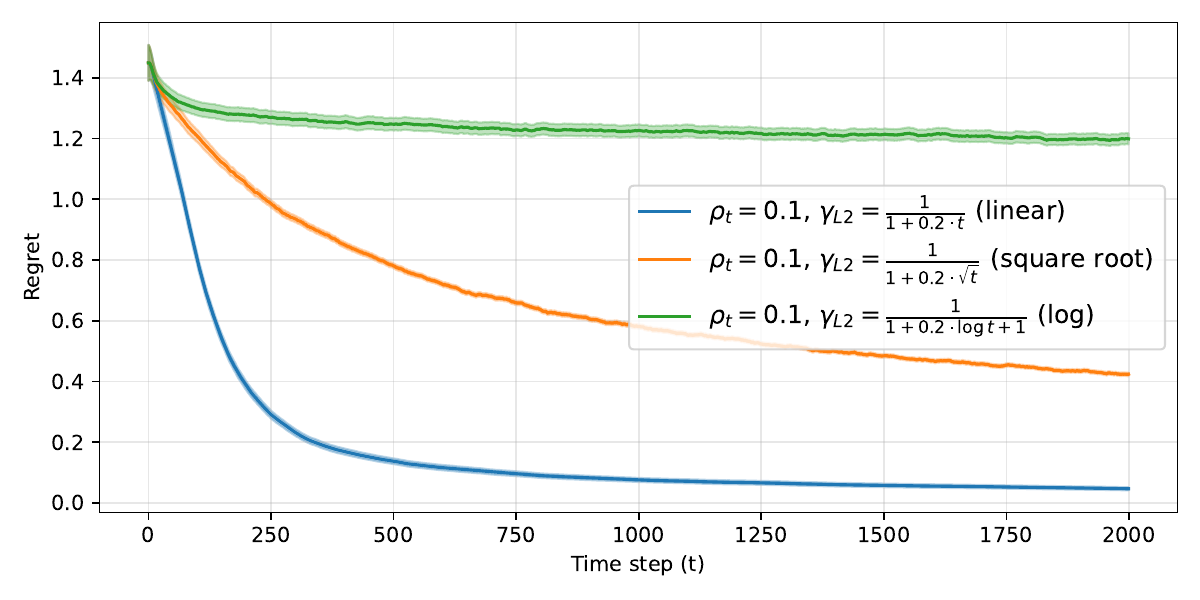}
\caption{The average regret and 95\% CI when starting from   $H_0=(5,...,0)$, $\rho_t=0.1$ (constant);
 several decay regimes are considered for $\gamma_t$: linear, square root and logarithmic.}  \label{fig:biased_several_decays}
\end{center}
\end{figure}

We tested the regularized MAB for several decay schedules $\gamma_t \to 0$\footnote{See also \Cref{sec:comparison_ucb} for the scope of our experiments and additional results on UCB algorithm.}. The  Python code is available
at 
\url{https://github.com/gabriel-turinici/regularized_policy_gradient} version December 20th 2025.
See \Cref{sec:formulas_entropy} \Cref{eq:Lcal_entropy} for
the precise meaning of $\gamma_{L2}$ and $\gamma_{Ent}$.

As in \cite{sutton_reinforcement_2018}(fig 2.5 page 38) we perform 
$M=1000$ runs, each having $T=2000$ time steps with $k=10$ arms; averages 
$q_*(a)$, $a \in [k]$ are independent and 
 sampled from a normal distribution
 located at $4$ and unit variance. 
For each of the arms $a \in [k]$ the conditional distribution of $R(a)$ with respect to $a$ is a Gaussian of mean $q_*(a)$ and variance~$1$.

Each of the $M$ runs has its own $q_*(\cdot)$ which do not change during the $T$ steps of the run. To obtain coherent comparisons, we use the same values of $q_*(\cdot)$ for all the bandits that are plotted in the same figure. 

We plot the regret, 
cf. the  definition in \eqref{eq:def_regret}, 
the lower the regret the better the algorithm performs.

In all cases we take the starting point to be biased, here $H_0=(5,0,...,0)$ as this was shown to correspond to more difficult situations~\cite{mei_global_2020}. This is also motivated by possible application to transient MAB where the averages $q_*$ may shift in time.

\subsection{Usefulness of the L2 regularization}

\subsubsection{Baseline}
\label{sec:baseline}

We first check that the regularization is useful by comparing the convergence with or without regularization; in order to 
make the addition of the regularization more clear we 
set as non-regularized baseline the best performing 
order of magnitude 
of vanilla policy gradient MAB
by looking on a grid spanning from $\rho=0.01$ to $\rho=100$;
the results are in Figure~\ref{fig:constant_rho_no_reg} and show that $\rho_t=0.1$ performs better than all the others.
This optimal value will be taken as baseline from now on.
We also give corresponding results for entropy regularization 
in \Cref{fig:rho01_cst_gamma001to100_entropy}
and see that the best constant is in the range $0.01-1$, with lowest values being better when $t$ is large; accordingly, we take for our entropy the schedule $\gamma_t = \frac{1}{1 + 0.2 t}$ to span all optimal ranges. Further baseline considerations are to be found in the Appendix, including a more extensive grid search for linear time schedules of the form $\rho_t = \frac{c_1}{1+c_2\cdot t}$, see \Cref{fig:grid_search_linear_schedule_rhot}.

\subsubsection{Results with respect to baseline}

We now compare the baseline with a regularized version in 
Figure~\ref{fig:cst_rho_compare_L2_reg_entropy}; the decay schedule chosen for the regularization is linear $\gamma_t = \frac{\gamma_0}{1 + c_1 t}$, $c_1 >0$. The results show that the presence of a regularization term helps converge faster and displays faster decay in the early phases.
Compared to entropy, the L2 regularization seems to be
the only way to consistently outperform the baseline. 

We also tested the robustness with respect to the optimality of $\rho_t$; we take a value that is the second best in \Cref{fig:constant_rho_no_reg} i.e., $\rho_t=1$; in this case the difference is even more striking, as seen in Figure~\ref{fig:cst_rho1_compare_L2_reg_entropy}, where the result indicates that the L2 regularization is able to 'repair' the non-optimal $\rho_t$ value and make it converge as the optimal one.
This shows that the L2 regularization is robust to other parameters and works towards improving the outcome.

\subsubsection{Several decay schedules}

We next compare several decay schedules in Figure~\ref{fig:biased_several_decays}: linear, logarithmic and square root. Based on these results the linear schedule appears to be the best, but other test cases may yield different results  depending on the parameter set.

\subsection{Robustness with respect to number of arms}

\begin{figure}[ht]
  \begin{center}
    \centerline{
\includegraphics[width=\columnwidth]{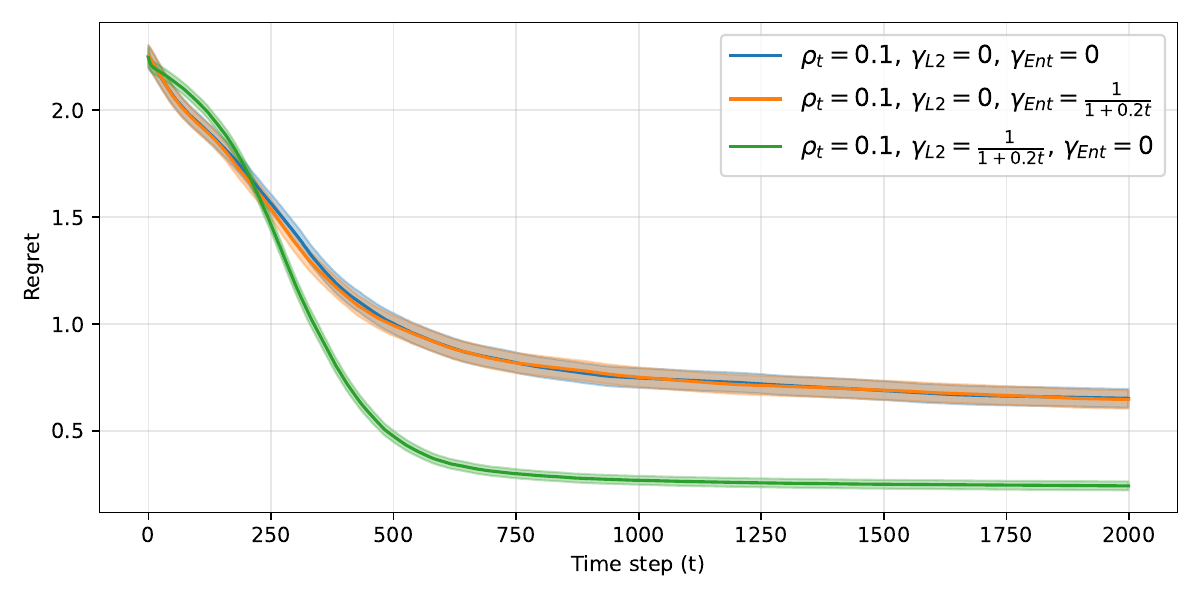}
}
\caption{The 
analogue of \Cref{fig:cst_rho_compare_L2_reg_entropy} except that now we have $k=50$ arms.
The L2 regularization still has the lowest regret among all scenarios considered: no regularization, L2 regularization, entropy regularization.}  \label{fig:cst_rho_compare_L2_reg_entropyk50}
  \end{center}
\end{figure}

\begin{figure}[ht]
  \begin{center}
    \centerline{
\includegraphics[width=\columnwidth]{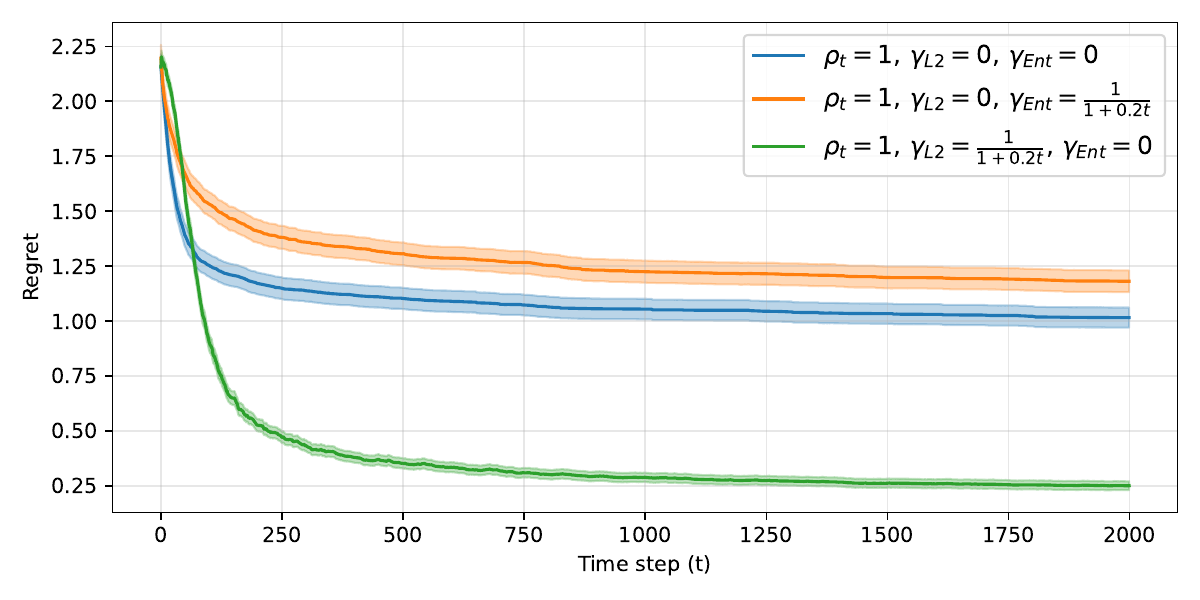}
}
\caption{The  analogue of \Cref{fig:cst_rho1_compare_L2_reg_entropy} for  $k=50$ arms.} 
\label{fig:cst_rho1_compare_L2_reg_entropyk50}
  \end{center}
\end{figure}

\begin{figure}[ht]
  \begin{center}
    \centerline{
\includegraphics[width=\columnwidth]{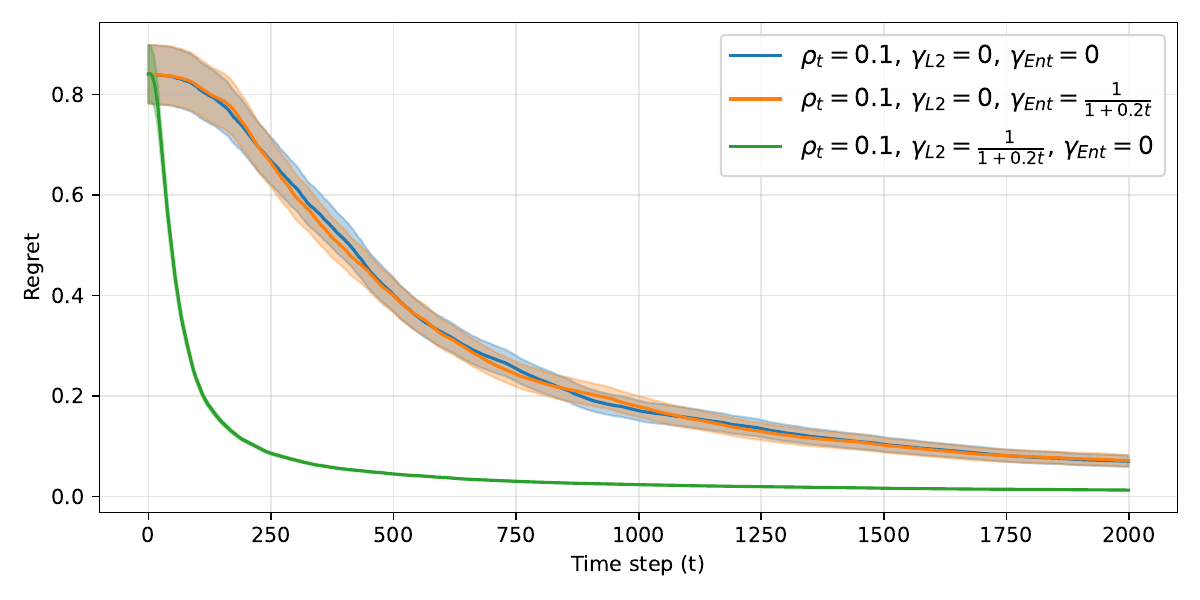}
}
\caption{The 
analogue of \Cref{fig:cst_rho_compare_L2_reg_entropy,fig:cst_rho_compare_L2_reg_entropyk50} for $k=3$ arms.
The convergence is faster but L2 regularization  manages to obtain lowest regret after fewer steps than the other two methods.}  \label{fig:cst_rho_compare_L2_reg_entropyk3}
  \end{center}
\end{figure}

To ensure the results are not specific to a given setting, we varied the number of arms. The results when we increased $k$ to $50$ are presented in \Cref{fig:cst_rho_compare_L2_reg_entropyk50}, 
which is the analogous to 
\Cref{fig:cst_rho_compare_L2_reg_entropy}
and 
\Cref{fig:cst_rho1_compare_L2_reg_entropyk50}
which is analogous to 
\Cref{fig:cst_rho1_compare_L2_reg_entropy}.
We observe in \Cref{fig:cst_rho_compare_L2_reg_entropyk50} that L2 regularization is still performing best among all scenarios.

As a comparison we also tested the results for $k=3$ arms. It is no surprise that in this case the convergence is faster and the no-regularization and entropic regularization manage to do a good job at it, as we can see in 
\Cref{fig:cst_rho_compare_L2_reg_entropyk3} but $L2$ regularization is still outperforming both.

\subsection{Robustness with respect to the reward distribution}
\label{sec:numerical_student}

\begin{figure}[ht]
  \begin{center}
    \centerline{
\includegraphics[width=\columnwidth]{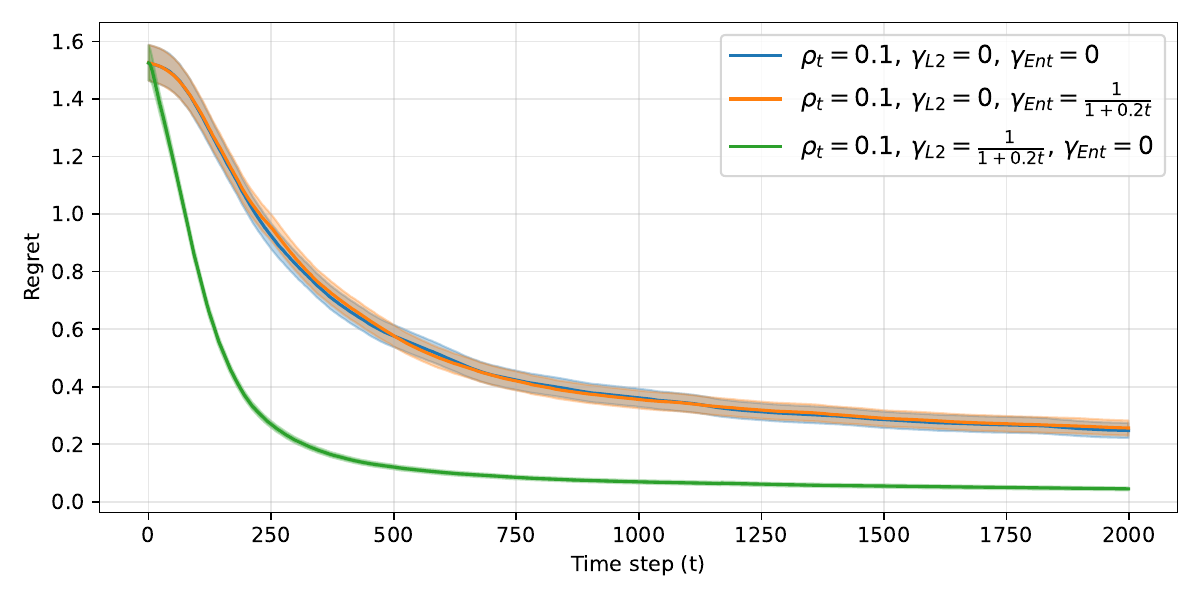}
}
\caption{The 
analogue of \Cref{fig:cst_rho_compare_L2_reg_entropy} except that now the stochasticity is given by a Student-t distribution.
with parameter $\nu=2.5$. The L2 regularization still has the lowest regret among all scenarios considered.}  \label{fig:cst_rho_compare_L2_reg_entropyk10t}
  \end{center}
\end{figure}

\begin{figure}[ht]
  \begin{center}
    \centerline{
\includegraphics[width=\columnwidth]{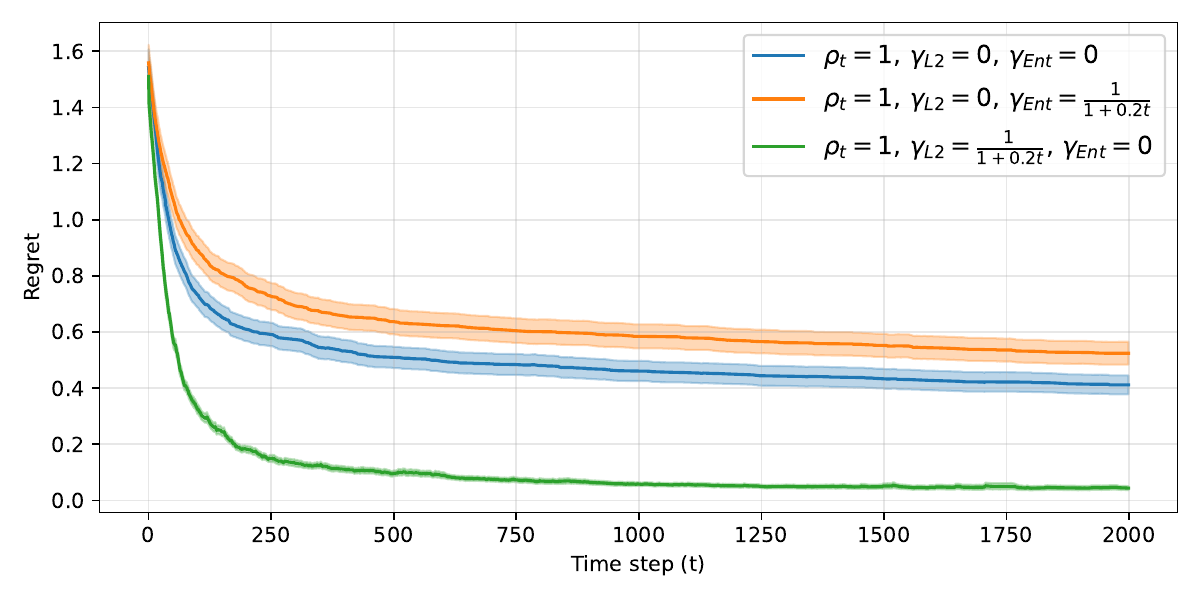}
}
\caption{The  analogue of \Cref{fig:cst_rho1_compare_L2_reg_entropy} when the rewards are drawn from a Student-t distribution with parameter $\nu=2.5$.} 
\label{fig:cst_rho1_compare_L2_reg_entropyk10t}
  \end{center}
\end{figure}

\begin{figure}[ht]
  \begin{center}
    \centerline{
\includegraphics[width=\columnwidth]{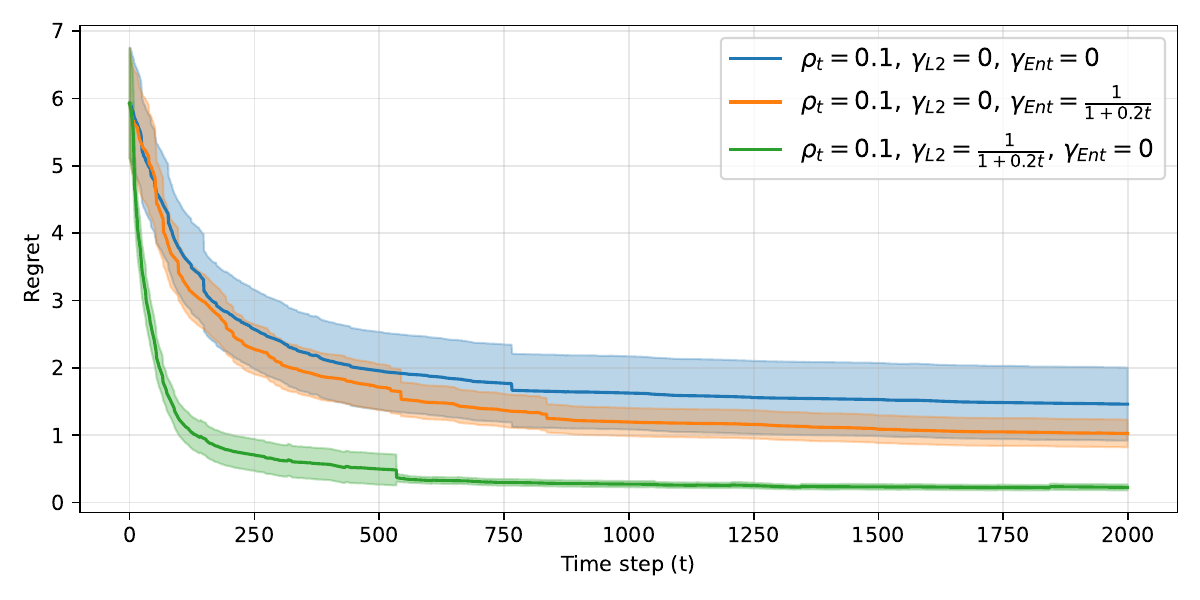}
}
\caption{The 
analogue of \Cref{fig:cst_rho_compare_L2_reg_entropy} except that now the stochasticity is given by a heavy tailed Student-t distribution 
with parameter $\nu=1.5$. The L2 regularization still has the lowest regret among all scenarios considered.}  \label{fig:cst_rho_compare_L2_reg_entropyk10t15}
  \end{center}
\end{figure}

One of the important characteristics of the theoretical results is that it does not requires the support of the random distributions to be bounded. Accordingly, we test here robustness with respect to the heavy-tailed nature of the distribution and consider the situation where the rewards are drawn from a Student-t distribution with finite second order moments but infinite third order moments i.e.
 has  $\nu=2.5$ degrees of freedom. In order to be consistent with previous tests we also rescaled by constant  $\sqrt{(\nu - 2.0) / \nu}$ so that is has unit variance. 
 The result in \Cref{fig:cst_rho_compare_L2_reg_entropyk10t}
shows that the L2 regularization works well even for such situations. 

A equally spectacular difference appears in 
\Cref{fig:cst_rho1_compare_L2_reg_entropyk10t}; as in 
\Cref{fig:cst_rho1_compare_L2_reg_entropy}
we took here some non-optimal $\rho_t=\rho$ and let L2 regularization 'correct' it. We see that the plain softmax gradient and the entropy-regularized gradients are simply not converging, stagnating at a regret of around $0.5$ which, at this scale, is rather important. On the contrary the L2 regularization helps restore the convergence.

Finally, we pushed the experiments even further and tested a Student-t distribution with parameter $\nu=1.5$\footnote{Note that for this value of $\nu$ the distribution does not even have finite variance.}; the result in \Cref{fig:cst_rho_compare_L2_reg_entropyk10t15} confirms the good performance of the $L2$ regularization  for such  heavy tailed distributions.

\section{Summary and discussion} \label{sec:conclusion}

We analyze in this work 
the $L2$ regularized softmax policy gradient for MAB. 
The regularization term is parameterized by a multiplicative coefficient  $\gamma_t$ and  we consider for the first time the vanishing regime $\gamma_t \searrow 0$; we present both theoretical (convergence) and numerical (performance) results.

We prove that when $\gamma_t \searrow 0$ the gradient of the average reward vanishes, that is, critical point equations are satisfied in the limit.
Unlike prior work on non-regularized MAB, our analysis does not rely on any bounded-support assumption on the arm reward distributions
$R(a)$, $a \in [k]$.
This is the first result known for this setting. During the proof we highlight an interesting new regime, cf. \eqref{eq:gamma_hyp2}, that links the main learning rate $\rho_t$ with the regularization decay $\gamma_t$; this regime goes beyond classical Robbins-Monro setting that would correspond to $\sum_t \rho_t \gamma_t < \infty$ and may be interesting to explore further. We also prove convergence under hypotheses coherent with the literature.

The theoretical results were tested numerically; it is seen that the L2 regularization improves the convergence when compared with non-regularized or entropy-regularized alternatives. The improvement is robust across a large range of conditions including the number of arms and the distribution of the rewards. In particular we also tested successfully heavy tailed distributions (Student-t with $\nu=2.5$ or $\nu=1.5$) which are beyond the scope of the previous known results even for the non-regularized setting.

\clearpage 
\newpage

\bibliographystyle{icml2026}
\bibliography{refs}

\newpage
\appendix
\onecolumn
\section{Appendix: proofs and further numerical results}

\subsection{Proof of Lemma~\ref{lemma:bounded_gH_gt}}
\begin{proof}
We take  $t\ge t_0$ and first prove the inequality \eqref{eq:boundHzero}. Recall the update formula
\eqref{eq:definition_Ht_pg} for $H_{t+1}$ : $H_{t+1} = (1-\rho_t \gamma_t) H_t + \rho_t u_t$. We multiply by $\gamma_t$ to obtain
\begin{align}
     \gamma_t H_{t+1} = (1-\rho_t \gamma_t) \gamma_t H_t + \gamma_t \rho_t u_t,
\end{align}
and thus~:
\begin{align}
& \E[\| \gamma_{t+1} H_{t+1}\|^2 ] \le  
\E[\| \gamma_{t} H_{t+1}\|^2]  
\nonumber\\ & 
= \E[ \| (1-\rho_t \gamma_t) \gamma_t H_t + \gamma_t \rho_t u_t \|^2 ]
\nonumber\\ & 
\le (1-\rho_t \gamma_t) \E [\| \gamma_t H_t \|^2] + \gamma_t \rho_t \E [\| u_t \|^2],
\end{align}
where the last inequality is due to convexity; together with \eqref{eq:bounded_gradu} this leads to
\begin{align}
\E[\| \gamma_{t+1} H_{t+1}\|^2 ] \le \max \{ \E[\| \gamma_{t} H_{t}\|^2 ], c_{q_*}^2 \},    
\end{align} which ends the proof of \eqref{eq:boundHzero}.
The inequality \eqref{eq:bounded_grad_cg} is then a consequence of
\eqref{eq:bounded_grad} and   
\eqref{eq:boundHzero}.
\end{proof}

\subsection{Boundedness result when  $\lim_{t\to\infty} \gamma_t > 0$}

\begin{lemma}
	Under assumptions~	\eqref{eq:hyp_second_moment}, \eqref{eq:hyp_rhot} 
and 
\begin{align}
\textbf{Hyp}^{>0}_\gamma: \ \ \ 
	\gamma_t \text{ is a decreasing sequence, } \ \lim_{t\to\infty} \gamma_t = \bar{\gamma}> 0,
	\label{eq:hyp_gammabar}
\end{align}
 denote $t_0$ the first index such that for all $t\ge t_0$: $\rho_t \gamma_t \le 1$. Then:
\begin{align}
&\forall t \ge t_0: 
\E[\|H_t \|^2] \le \max\{ \E[\|H_{t_0} \|^2], c_{q_*}^2/ \bar{\gamma}^2 \}.
\label{eq:boundHnonzero}
\end{align}
Moreover there exists a constant $C_g' >0$ such that~:
\begin{equation}
\E [ \|g_t \|^2 ] \le C_g'.
\label{eq:bounded_grad_cg2}
\end{equation}
\label{lemma:boundedness_gammabarnonzero}
\end{lemma}
\begin{proof}
For 
\eqref{eq:boundHnonzero} the proof is similar 
to that of \eqref{eq:boundHzero}
as we obtain by convexity that~:
$\E[\|H_{t+1}\|^2 ] \le
\max \{ \E[\| H_{t}\|^2 ], c_{q_*}^2 /\gamma_t^2\} \le 
\max \{ \E[\| H_{t}\|^2 ], c_{q_*}^2 /\bar{\gamma}^2\}
$ and conclude as before.
Inequality~\eqref{eq:bounded_grad_cg2} results from  \eqref{eq:bounded_grad} and the boundedness of $\E[\| H_{t}\|^2 ]$.
\end{proof}

\subsection{Proof of Proposition \ref{prop:summability}}
\begin{proof}
It was noted that cf. \cite{anicta2024convergence}(Eqn (29))\footnote{For any operator $\Fcal: \R^k \to \R$ we denote by $\nabla_H^2 \Fcal$
its second differential which acts from $\R^k\times \R^k$ to $\R$ by:
$\nabla_H^2 \Fcal(H_1,H_2) = H_1^T \Hcal_\Fcal H_2$ where $\Hcal_\Fcal$ is the Hessian matrix of $\Fcal$. When there is no ambiguity we also write 
$\nabla^2 \Fcal$.
}
denoting $c_\star= \frac{\max_a q_*(a) + \min_a q_*(a)}{2} >0$~: 
\begin{align} 
| \nabla_H^2{\cal L}_0(\bar{H})(\delta H,\delta H)|\leq c_\star\| \delta H\|^2, \quad \forall \delta H, \bar{H}\in \mathbb{R}^k.  
\end{align}
\noindent
Taylor's formula allows to write~:
\begin{align}
 {\cal L}_{\gamma _t}(H_{t+1})-{\cal L}_{\gamma _t}(H_t)-\langle \nabla _H {\cal L}_{\gamma _t}(H_t),H_{t+1}-H_t\rangle
\displaystyle =\frac{1}{2}(H_{t+1}-H_t)^T\nabla _H^2{\cal L}_{\gamma _t}(\bar{H}_t)(H_{t+1}-H_t),    
\end{align}
\noindent
and consequently
\begin{align}
\rho_t\langle \nabla _{H}{\cal L}_{\gamma _t}(H_t), u_t-\gamma _tH_t\rangle -\frac{c_\star}{2}\|H_{t+1}-H_t\|^2\leq 
{\cal L}_{\gamma _t}(H_{t+1})-{\cal L}_{\gamma _t}(H_t).
\end{align}

\noindent
That implies
\begin{align}
-c_\star[\rho _t^2\| u_t\|^2+\rho _t^2\gamma _t^2\| H_t\|^2]+\rho_t\langle \nabla _{H}{\cal L}_{\gamma _t}(H_t), u_t-\gamma _tH_t\rangle 
\leq {\cal L}_{\gamma _{t+1}}(H_{t+1})-{\cal L}_{\gamma _t}(H_t)+\frac{\gamma _{t+1}-\gamma _t}{2}\| H_{t+1}\|^2
.\end{align}

\noindent
Using the conditional expectation with respect to ${\cal F}_t$ and then $\mathbb{E}$ we get that
\begin{align}
    &
  -\rho _t^2C_1+\rho _t\mathbb{E}[\|\nabla _H{\cal L}_{\gamma _t}(H_t)\|^2]
+\frac{\gamma _t-\gamma _{t+1}}{2}\mathbb{E}[\| H_{t+1}\|^2]
\le \mathbb{E}[{\cal L}_{\gamma _{t+1}}(H_{t+1})]-\mathbb{E}[{\cal L}_{\gamma _t}(H_t)],  
\end{align}
for any $t\in \mathbb{N}$, where $C_1$ is some positive constant. 
We obtain
\begin{align}
\rho _t\mathbb{E}[\|\nabla _H{\cal L}_{\gamma _t}(H_t)\|^2]
+\frac{\gamma _t-\gamma _{t+1}}{2}\mathbb{E}[\| H_{t+1}\|^2]
\leq \mathbb{E}[{\cal L}_{\gamma _{t+1}}(H_{t+1})]-\mathbb{E}[{\cal L}_{\gamma _t}(H_t)]+C_1 \rho_t^2.    
\label{eq:inequality_nabla_grad_proof1}
\end{align}
Since $\{ \mathbb{E}[{\cal L}_{\gamma _t}(H_t)]\} _{t\in \mathbb{N}}$ is bounded, $\gamma_t$ decreasing and $\rho_t^2$ summable 
we obtain conclusions~\eqref{eq:summability} and~\eqref{eq:summability_dgammaH2}. 
For \eqref{eq:summability_dgammaH2t} we use that $E[\|H_{t+1}-H_t\|^2] \le \rho^2_t C_g$ which is also summable. 
Since on the other hand 
Hypothesis~\eqref{eq:hyp_rhot} implies that $\Lambda_T \to \infty$ so that 
$1/\Lambda_T \to 0$ we obtain also~\eqref{eq:meancv}. Finally, arguing by contradiction we also get that on a subsequence
\begin{align}
  \mathbb{E}[\|\nabla _H{\cal L}_{\gamma _{t_l}}(H_{t_l})\|^2]\longrightarrow 0,
\end{align}
i.e.
$\nabla _H{\cal L}_{\gamma _{t_l}}(H_{t_l})\longrightarrow 0$ in $L^2$, 
which is another way to express \eqref{eq:liminf}.
\end{proof}

\subsection{Lemma~\ref{lemma:seriesXnYn}}

\begin{lemma}
Let $(X_n)_{n\ge 1}$ and $(Y_n)_{n\ge 1}$ be sequences of real-valued random variables with 
$
\sum_{n=1}^\infty \mathbb{E}\!\left[\,|X_n|^2\,\right] < \infty$, 
$
\sum_{n=1}^\infty \mathbb{E}\!\left[\,|Y_n|^2\,\right] < \infty$.
 Then
$
\sum_{n=1}^\infty \mathbb{E}[|X_n Y_n|] < \infty
$.
\label{lemma:seriesXnYn}
\end{lemma}
\begin{proof} We repeatedly use the Cauchy--Schwarz inequality:
\begin{align}
& \sum_{n=1}^\infty \mathbb{E}[|X_n Y_n|]
\le 
\sum_{n=1}^\infty 
\sqrt{\mathbb{E}[X_n^2]\,\mathbb{E}[Y_n^2]}
\le 
\left( \sum_{n=1}^\infty \mathbb{E}[X_n^2] \right)^{1/2}
\left( \sum_{n=1}^\infty \mathbb{E}[Y_n^2] \right)^{1/2} < \infty.
\end{align}
\end{proof}

\subsection{Summability Lemma~\ref{lemma:summability_seq_cv}}

\begin{lemma}
    Let $(\alpha_n)_{n\ge 1} \subset \R$ be a sequence with 
    $\alpha_n \ge 0$ and assume that there exists a
    sequence $(\beta_n)_{n\ge 1} \subset \R$  such that
\begin{align}
\forall n \ge 1: \ 
    \alpha_{n+1} - \alpha_n \le \beta_n.
\label{eq:alphabeta}
\end{align}
If $\sum_{n = 1}^\infty |\beta_n| < \infty$
then the sequence $(\alpha_n)_{n\ge 1}$ converges to some $\bar{\alpha}\in [0, \alpha_1+ \sum_{n = 1}^\infty \beta_n]$.
\label{lemma:summability_seq_cv}
\end{lemma}
\begin{proof}
Summing the inequalities in ~\eqref{eq:alphabeta} we obtain that 
$\alpha_n- \alpha_1 \le \sum_{k=1}^n \beta_k < \infty$ thus the sequence $(\alpha_n)_{n\ge 1}$  is bounded from above by 
$\alpha_1+ \sum_{n = 1}^\infty |\beta_n| < \infty$ and from below by $0$. Assume that 
is has two distinct accumulation points $\bar{\alpha}_1<\bar{\alpha}_2$ and denote $\epsilon = \bar{\alpha}_2-\bar{\alpha}_1$. Since the series of general term $\beta_n$ is absolutely convergent, there exists $n_0$ such that for any $n_1, n_2 \ge n_0$
$\sum_{n= n_1}^{n_2} |\beta_n| <  \frac{\epsilon}{2}$. Since $\bar{\alpha}_1$ is an accumulation point there exists some $n_1> n_0$  such that $|\alpha_{n_1}- \bar{\alpha}_1 | < \epsilon/4$
and some $n_2 > n_1$ such that 
$|\alpha_{n_2}- \bar{\alpha}_2 | < \epsilon/4$. But then
\begin{align}
   & \epsilon =  \bar{\alpha}_2-\bar{\alpha}_1 \le 
   |\bar{\alpha}_2-\alpha_{n_2} | +  |\alpha_{n_2}-\alpha_{n_1}| +
   |\alpha_{n_1}- \bar{\alpha}_1 |
   < \frac{\epsilon}{4} + \sum_{n= n_1}^{n_2-1} |\beta_n| + \frac{\epsilon}{4} < \epsilon,
\end{align}
which is a contradiction. Therefore there is only one accumulation point and, since the sequence is bounded, it converges.
\end{proof}

\subsection{Proof of Theorem~\ref{thm:cv_nabla}}

\begin{proof}
\noindent 
{\bf Proof of assertion~\eqref{eq:limLgammat}:} 
Denote $G_t= \| \nabla\mathcal{L}_{0}(H_{t})-\gamma_{t}H_{t} \|^2$. 
To analyze the difference  $G_{t+1}-G_t$ we invoke the identity:
$    ||\zeta||^{2}-||\nu||^{2}=||\zeta-\nu||^{2} + 2\langle \zeta-\nu, \nu\rangle$,  true for any vectors $\zeta$, $\nu$ ; we choose:
$\zeta = \nabla \mathcal{L}_{0}(H_{t+1})-\gamma_{t+1}H_{t+1}$ and $\nu = \nabla \mathcal{L}_{0}(H_{t})-\gamma_{t}H_{t}$ and thus:
\begin{align}
&G_{t+1}-G_t
=
\underbrace{
||\nabla \mathcal{L}_{0}(H_{t+1})-\nabla \mathcal{L}_{0}(H_{t})-(\gamma_{t+1}H_{t+1}-\gamma_{t}H_{t})||^{2}}_
{
\text{term (I)}
}  \nonumber \\&
+ 2 \ \ \underbrace{\langle 
\nabla \mathcal{L}_{0}(H_{t+1})-\nabla \mathcal{L}_{0}(H_{t})-(\gamma_{t+1}H_{t+1}-\gamma_{t}H_{t}), \nabla \mathcal{L}_{0}(H_{t})-\gamma_{t}H_{t} \rangle 
}_{\text{term (II)}}.
\label{eq:diff_Gt}
\end{align}
For the term (I) we obtain:
\begin{align}
&  \E[  ||\nabla \mathcal{L}_{0}(H_{t+1})-\nabla \mathcal{L}_{0}(H_{t})-(\gamma_{t+1}H_{t+1}-\gamma_{t}H_{t})||^{2} ]
\nonumber \\  &
\le
3 \E \left\{||\nabla\mathcal{L}_{0}(H_{t+1})-\nabla\mathcal{L}_{0}(H_{t})||^{2} 
+(\gamma_{t}-\gamma_{t+1})^{2}||H_{t+1}||^{2}+\gamma_{t}^{2}||H_{t}-H_{t+1}||^{2} \right\}
\nonumber \\  &
\le
3 \E \left\{ (L_{\nabla\mathcal{L}_{0}} + \gamma_0^2) ||H_{t}-H_{t+1}||^{2} + \gamma_0 (\gamma_{t}-\gamma_{t+1})||H_{t+1}||^{2}
\right\}, \label{eq:delta_gt}
\end{align}
where $L_{\nabla\mathcal{L}_{0}}$ is the Lipschitz constant of 
$\nabla\mathcal{L}_{0}(H) $, which is finite because if only involves components of $\Pi _H$. 
The second term  of the last expression 
is member of a summable series as proved  in \eqref{eq:summability_dgammaH2}. The first term 
$ \E[(L_{\nabla\mathcal{L}_{0}} + \gamma_0^2) ||H_{t}-H_{t+1}||^{2}]$ is the same as $ \rho_t ^2 \E[(L_{\nabla\mathcal{L}_{0}} + \gamma_0^2) ||g_{t}||^{2}] $ and is summable because $\E[||g_t||^{2}] $ is bounded as proven in~\eqref{eq:bounded_grad_cg} and  $\sum_t \rho_t^2 < \infty$ (hypothesis $\textbf{Hyp}_\rho$ in~\eqref{eq:hyp_rhot}). Thus the term (I) is a member of a summable series.

\noindent Let us now discuss term (II) in \eqref{eq:diff_Gt} that can be decomposed as 
\begin{align}
&\langle 
\nabla \mathcal{L}_{0}(H_{t+1})-\nabla \mathcal{L}_{0}(H_{t})-(\gamma_{t+1}H_{t+1}-\gamma_{t}H_{t}), \nabla \mathcal{L}_{0}(H_{t})-\gamma_{t}H_{t} \rangle
\nonumber \\&
=\underbrace{\langle
\nabla \mathcal{L}_{0}(H_{t+1})-\nabla \mathcal{L}_{0}(H_{t}), 
\nabla \mathcal{L}_{0}(H_{t})-\gamma_{t}H_{t} \rangle}_{\text{term (*)}} 
\nonumber \\&
+
\underbrace{
\langle \gamma_{t}H_{t}-\gamma_{t+1}H_{t+1}, \nabla \mathcal{L}_{0}(H_{t})-\gamma_{t}H_{t} \rangle}_{\text{term (**)}}  ~.
\end{align}
For the term (*) we can write~:
\begin{align}
& \E[
\langle
\nabla \mathcal{L}_{0}(H_{t+1})-\nabla \mathcal{L}_{0}(H_{t}), 
\nabla \mathcal{L}_{0}(H_{t})-\gamma_{t}H_{t} \rangle
] 
\nonumber \\&
=\E[
\langle
\nabla^2 \mathcal{L}_{0}(H_t) (H_{t+1}-H_t) 
\nonumber \\ &
+ 
\frac{1}{2}\nabla^3 \mathcal{L}_{0}(\bar{H}) (H_{t+1}-H_t,H_{t+1}-H_t), 
\nabla \mathcal{L}_{0}(H_{t})-\gamma_{t}H_{t} \rangle
] 
\nonumber \\&
=\E[\E[
\langle
\nabla^2 \mathcal{L}_{0}(H_t) (H_{t+1}-H_t), 
\nabla \mathcal{L}_{0}(H_{t})-\gamma_{t}H_{t} \rangle | \Fcal_t
] 
]
\nonumber \\&
+\E[
\langle
\frac{1}{2}\nabla^3 \mathcal{L}_{0}(\bar{H}) (H_{t+1}-H_t,H_{t+1}-H_t), 
\nabla \mathcal{L}_{0}(H_{t})-\gamma_{t}H_{t} \rangle
] 
\nonumber \\&
= \rho_t \E[
\langle
\nabla^2 \mathcal{L}_{0}(H_t) (\nabla \mathcal{L}_{0}(H_{t})-\gamma_{t}H_{t}), 
\nabla \mathcal{L}_{0}(H_{t})-\gamma_{t}H_{t} \rangle 
]
\nonumber \\&
+\E[
\langle
\frac{1}{2}\nabla^3 \mathcal{L}_{0}(\bar{H}) (H_{t+1}-H_t,H_{t+1}-H_t), 
\nabla \mathcal{L}_{0}(H_{t})-\gamma_{t}H_{t} \rangle
] 
\nonumber \\&
\le c_1 \rho_t \E[\|\nabla \mathcal{L}_{0}(H_{t})-\gamma_{t}H_{t}\|^2]
+ 
c_2 \E[ \rho_t ^2
\|g_t\|^2 \cdot \|
\nabla \mathcal{L}_{0}(H_{t})-\gamma_{t}H_{t} \|
], 
\end{align}
with $c_1$ and $c_2$ some constants depending only on $q_*$.
In the last expression, the first part is the general term of a convergent series as proven in \eqref{eq:summability}. The second part is also summable because $\rho_t^2$ is summable and 
$\E[
\|g_t\|^2 \cdot \|
\nabla \mathcal{L}_{0}(H_{t})-\gamma_{t}H_{t} \|
] \le
\E[
\|g_t\|^3]^{2/3} \cdot \E[\|
\nabla \mathcal{L}_{0}(H_{t})-\gamma_{t}H_{t} \|^3
]^{1/3}$ is bounded because of the hypothesis \eqref{eq:hyp_third_moment}\footnote{We also use a variant of Lemma~\ref{lemma:bounded_gH_gt} which ensures that the third order moment of $g_t$ and $\gamma_{t}H_{t}$ are both bounded; the proof being totally analogous to that before we do not give it here.}.
Thus term (*)  is absolutely summable. 

\noindent
We are now left with term (**) that can be further decomposed as~:
\begin{align}
    \underbrace{
\langle \gamma_{t+1}(H_{t}-H_{t+1}), \nabla \mathcal{L}_{0}(H_{t})-\gamma_{t}H_{t} \rangle}_{\text{term }(\dagger)}  
+
    \underbrace{
\langle (\gamma_{t}-\gamma_{t+1})H_{t}, \nabla \mathcal{L}_{0}(H_{t})-\gamma_{t}H_{t} \rangle}_{\text{term }(\ddagger)}.  
\end{align}
The average of  
($\dagger$) is non-positive, while, arguing as before, the average of ($\ddagger$) is upper bounded by $(\gamma_t - \gamma_{t+1}) \cdot \frac{\E [ \| H_{t}\|^2] + 
\E[ \| \nabla \mathcal{L}_{0}(H_{t})-\gamma_{t}H_{t}  \|^2] }{2}$ which is summable.
Coming back to (\ref{eq:diff_Gt}) we can summarize the whole analysis by saying that $\E[G_{t+1}] - \E[G_t]$ is upper bounded by 
the general term of an absolutely summable series. Now we can apply Lemma~\ref{lemma:summability_seq_cv}
to conclude that $\E[G_t]$ is a converging sequence. But, from~\eqref{eq:liminf} its liminf is zero so the whole sequence is converging to zero, 
which ends the proof of assertion~\eqref{eq:limLgammat}.

\noindent
{\bf Proof of assertions \eqref{eq:summability_L0gtHt}, \eqref{eq:gtHt2_cv}, \eqref{eq:lim_gammatHt} and \eqref{eq:limL0}} (\eqref{eq:limLgammat} is a consequence of 
and  \eqref{eq:lim_gammatHt} and \eqref{eq:limL0}): 
From the summability \eqref{eq:summability_dgammaH2t} and relation \eqref{eq:gamma_hyp2} we obtain that 
\begin{align}
 \sum_{t=0}^{\infty } \rho_t \gamma_t^2\mathbb{E}[ \| H_{t}\|^2]<\infty
 \label{eq:summabilitygHH},
\end{align}
and then, since 
$\nabla _H{\cal L}_{\gamma _t}(H_t) = \nabla _H{\cal L}_0(H_t) -{\gamma _t}H_t$ we recover 
\eqref{eq:summability_L0gtHt} from \eqref{eq:summability}.
Denote
\begin{align}
 \Gcal(H)= \| \nabla\mathcal{L}_{0}(H)\|^2.   
\end{align}
The gradient $\nabla_H \Gcal$ is a Lipschitz function of $H$, we  denote $L_\Gcal$ its Lipschitz constant, and thus
\begin{align}
&   |\E [\Gcal(H_{t+1})] -\E [\Gcal(H_{t})] |\le  \E [\langle \nabla  \left\{ \|\nabla  \mathcal{L}_{0}(H_t)\|^2 \right\}, H_{t+1}-H_t\rangle] 
\nonumber \\ &
+ L_\Gcal \E [\| H_{t+1}-H_t \|^2]
\nonumber \\ &
=  \E [ \E [ \langle\nabla  \left\{ \|\nabla  \mathcal{L}_{0}(H_t)\|^2 \right\}, H_{t+1}-H_t\rangle |\Fcal_t] ]+ L_\Gcal \E [\| H_{t+1}-H_t \|^2]
\nonumber \\ &
= \rho_t \E [\langle \nabla  \left\{ \|\nabla  \mathcal{L}_{0}(H_t)\|^2 \right\},  \nabla \mathcal{L}_{0}(H_t)-\gamma_t H_t\rangle ]+ L_\Gcal \E [\| H_{t+1}-H_t \|^2]
\nonumber \\ &
\le c_1 \rho_t \E [\|\nabla  \mathcal{L}_{0}(H_t)\| \cdot \| \nabla \mathcal{L}_{0}(H_t)-\gamma_t H_t \| ]+ c_2 \rho_t^2,
\label{eq:in_proofL0cv}
\end{align}
where $c_1$ and $c_2$ are constants depending only on $q_*$.
Recall that series $\sum_t \rho_t^2$ is convergent, same as 
$\sum_t \rho_t \E [\|\nabla  \mathcal{L}_{0}(H_t)\|^2]$ 
and
$\sum_t \rho_t \E [ \| \nabla \mathcal{L}_{0}(H_t)-\gamma_t H_t \|^2]$.
By Lemma~\ref{lemma:seriesXnYn} we obtain that 
$ |\E [\Gcal(H_{t+1})] -\E [\Gcal(H_{t})] |$ is bounded by a convergent series and by Lemma~\ref{lemma:summability_seq_cv} we obtain that 
$\E[ \Gcal(H_{t})]$ is convergent. From  
\eqref{eq:summability_L0gtHt} the only possible limit is zero which proves \eqref{eq:limL0}.

\noindent 
To prove \eqref{eq:gtHt2_cv} we write:
\begin{align}
& \E [\gamma_{t+1} \|H_{t+1} \|^2] -\E [\gamma_{t}\|H_{t}\|^2] \le \E [ \gamma_{t} (\|H_{t+1} \|^2 - \|H_{t} \|^2) ]
\nonumber \\ &
= 
   \gamma_{t} \E \left[ (\|H_{t+1}  - H_{t} \|^2)  
+ 2 \langle H_{t+1}-H_{t},H_{t}\rangle \right]
\nonumber \\ &
=  \gamma_{t} \E \left[ (\|H_{t+1}  - H_{t} \|^2)  
+ 2 \rho_t \langle \nabla \mathcal{L}_{0}(H_t)-\gamma_t H_t,H_{t}\rangle 
 \right] \textrm{ (we used conditioning by} \Fcal_t \textrm{ )}
\nonumber \\ &
\le C_g \gamma_t \rho_t^2  + 2 \rho_t    \E [\| \nabla \mathcal{L}_{0}(H_t)-\gamma_t H_t \|\cdot \gamma_t \|H_t \|],
\end{align}
and, using 
 \eqref{eq:summabilitygHH}
we conclude as in \eqref{eq:in_proofL0cv} that $\E [\gamma_{t+1} \|H_{t+1} \|^2]$ converges. Since $\gamma_t \to 0$ we obtain \eqref{eq:lim_gammatHt}. Analogous arguments show that 
$\E[ \Lcal_0(H_t)]$ is also converging and hence $\E[ \Lcal_{\gamma_t}(H_t)]$ as difference of two converging sequences. 

We now turn to \eqref{eq:lim_in_Dcal}.  
Denote $\Delta$ the 'gap' $\Delta := \min_{a,b \in [k], a\neq b} |q_*(a) -q_*(b)|$;
from \eqref{eq:hyp_qa_diff_qb} we obtain $\Delta >0$.
Invoking formula (14) from \cite{mei2023stochastic} we can write
\begin{align}
\|\nabla_H \Lcal_0(H_t)\|^2= \sum_{a \in [k]} \Pi_{H_t}(a)^2 (q_*(a) - q_t)^2, 
\end{align}
where $q_t:= \langle q_*,\Pi_{H_t}\rangle$. 
Note that $q_t \in [\min_a q_*(a), \max_a q_*(a)]$ and denote $A_t$ one minimizer of $|q_*(a)-q_t|$  among all arms $a$; in particular $A_t$ is measurable thus a random variable. Moreover, since $\Delta$ is the minimal gap we obtain that $|q_*(a)-q_t| \ge \Delta/2$ whenever $a\neq A_t$. Thus:
\begin{align}
&     \|\nabla_H \Lcal_0(H_t)\|^2= \sum_{a\le k} \Pi_{H_t}(a)^2 (q_*(a) - q_t)^2 \ge  \sum_{a\neq A_t} \Pi_{H_t}(a)^2 (q_*(a) - q_t)^2 \ge \sum_{a\neq A_t} \Pi_{H_t}(a)^2 \left(\frac{\Delta}{2} \right)^2  
\nonumber \\ & 
\ge \left(\frac{\Delta}{2} \right)^2  \cdot \frac{1}{k-1} \cdot \left( \sum_{a\neq A_t} \Pi_{H_t}(a) \right)^2 = 
\left(\frac{\Delta}{2} \right)^2  \cdot \frac{1}{k-1}  \cdot (1-\Pi_{H_t}(A_t))^2 
\nonumber \\ & 
=
\frac{\Delta^2}{8(k-1)} \left[ 
(1-\Pi_{H_t}(A_t))^2 + (1-\Pi_{H_t}(A_t))^2
\right] = 
\frac{\Delta^2}{8(k-1)} \left[ 
(1-\Pi_{H_t}(A_t))^2 + \left(\sum_{a\neq A_t} \Pi_{H_t}(a)\right)^2
\right] 
\nonumber \\ & 
\ge
\frac{\Delta^2}{8(k-1)} \left[ 
(1-\Pi_{H_t}(A_t))^2 +\sum_{a\neq A_t} \Pi_{H_t}(a)^2
\right] 
= 
\frac{\Delta^2}{8(k-1)} \| \Pi_{H_t} - \delta_{A_t}\|^2
\overset{ \text{ from } \eqref{eq:def_dist_Dcal}}{\ge} 
\frac{\Delta^2}{8(k-1)} \| \Pi_{H_t} - \Dcal\|^2,
\label{eq:proof_piht_to_dcal_rhot}
\end{align}
which together with \eqref{eq:limL0} proves \eqref{eq:lim_in_Dcal}.  

To prove \eqref{eq:lim_regret_zero} we use Lemma 3 from \cite{mei_global_2020} that in our setting reads 
$\|\nabla \Lcal_0(H)\| \ge \Pi_H(a^*) \Rcal(\Pi_{H})$. 
Using \eqref{eq:hyp_liminf_piastar_positive} we obtain
$\|\nabla_H \Lcal_0(H_t)\| \ge c_0 \cdot \Rcal(\Pi_{H_t})$ which  implies  \eqref{eq:lim_regret_zero} because of \eqref{eq:limL0}.

To prove \eqref{eq:lim_qstar} we write for a general  distribution $\Pi \in \Pcal_k$:
\begin{align}
& \Rcal(\Pi)^2 = \left( \sum_{b\in [k]}   \Pi(b) (q_*(a^*) -q_*(b)) \right)^2 =
\left( \sum_{b\in [k], b \neq a^*}   \Pi(b) (q_*(a^*) -q_*(b)) \right)^2 
\nonumber \\ & 
\ge \Delta^2 \left( \sum_{b\in [k], b \neq a^*}   \Pi(b) \right)^2 = \frac{\Delta^2}{2} 
\left[ 
\left(\sum_{b\in [k], b \neq a^*}   \Pi(b) \right)^2 
+ \left( \sum_{b\in [k], b \neq a^*}   \Pi(b) \right)^2
\right]
\nonumber \\ & 
\ge
\frac{\Delta^2}{2} 
\left[ (1-\Pi(a^*))^2+
\sum_{b\in [k], b \neq a^*}   \Pi(b)^2 
\right] = \frac{\Delta^2}{2} \cdot
\|\Pi-\delta_{a^*} \|^2,
\label{eq:dem_rhot_regret_dist}
\end{align}
which, 
setting $\Pi=\Pi_{H_t}$ and 
using \eqref{eq:lim_regret_zero}, proves  \eqref{eq:lim_qstar}.
\end{proof}

\subsection{A stronger convergence result for $\Pi_{H_t}$ under more restrictive hypotheses}
\begin{lemma}
    Under assumptions~
\eqref{eq:hyp_second_moment}, 	
\eqref{eq:hyp_rhot},
\eqref{eq:hyp_gamma_decreasing}, \eqref{eq:hyp_qa_diff_qb}, and assuming additionally that the rewards are bounded and that
    \begin{equation}
        \sum_{t=1}^{\infty} \rho_t(\rho_0+\rho_1+\cdots +\rho_{t-1})^2\gamma_t^2 < \infty ,
    \end{equation}
then 
\begin{align}
\Pi_{H_t}\longrightarrow \left( \mathbbm{1}_{F^{-1}(q_*(a))} \right)_{a=1}^k \quad \text{in }  L^2,    
\label{eq:convergene_L2_Dirac_omega}
\end{align}
where
$F$ is a real-valued random variable satisfying $F\in\{q_*(1), q_*(2),..., q_*(k)\}$ a.s..
\label{lemma:cv_pi}
\end{lemma}

\begin{proof}
    Since
    $\|H_{t+1}\|\leq \|H_t\|(1-\rho_t\gamma_t)+M\rho_t$ ($M$ is some positive constant) then for any $t> t_0$: 
    \begin{equation}
    \begin{array}{rl}
      \|H_t\| & \leq  \|H_{t_0}\|(1-\rho_{t_0}\gamma_{t_0})(1-\rho_{t_0+1}\gamma_{t_0+1})\cdots(1-\rho_{t-1}\gamma_{t-1}) \\
     ~& \quad +M[\rho_{t-1}+\rho_{t-2}(1-\rho_{t-1}\gamma_{t-1})+\cdots +\rho_{t_0}(1-\rho_{t_0+1}\gamma_{t_0+1})\cdots(1-\rho_{t-1}\gamma_{t-1})] \\
     \Rightarrow \gamma_t\|H_t\| & \leq \gamma_t\|H_{t_0}\|+M\gamma_t(\rho_{t-1}+\rho_{t-2}+\cdots+\rho_0)\\
     \Rightarrow \gamma_t^2\|H_t\|^2 & \leq  \underbrace{\tilde{M}[\gamma_t^2 + \gamma_t^2(\rho_0+\cdots+\rho_{t-1})^2].}_{\text{$a_t$}}
    \end{array}
    \end{equation}
Here $\tilde{M}$ is another positive constant. 
So, $\gamma_t^2\|H_t\|^2 \leq a_t$ a.s., where $\sum_{t=1}^\infty a_t\rho_t < \infty$.

By Taylor's formula we get that
\begin{equation}
    \begin{array}{ll}
    \Lcal_0(H_{t+1}) & = \Lcal_0(H_t)+\langle\nabla\Lcal_0(H_t),H_{t+1}-H_t\rangle+\frac{1}{2}(H_{t+1}-H_t)^T\nabla_H^2\Lcal_0(\bar{H_t})(H_{t+1}-H_t) \\
    ~& \geq \Lcal_0(H_t)+\rho_t\langle\nabla\Lcal_0(H_t),u_t\rangle-\rho_t\gamma_t\langle\nabla\Lcal_0(H_t),H_t\rangle-\frac{c_\star}{2}\|\rho_t u_t-\rho_t\gamma_t H_t\|^2
    \end{array}
\end{equation}

\begin{equation*}
    \begin{array}{ll}
    \Rightarrow \mathbb{E}[\Lcal_0(H_{t+1})|\Fcal_t] & \geq \Lcal_0(H_t)+\rho_t\|\nabla\Lcal_0(H_t)\|^2-\rho_t\|\nabla\Lcal_0(H_t)\|^2-\rho_t \frac{\gamma_t^2}{4}\|H_t\|^2-c\rho_t^2 \\
    ~& \geq \Lcal_0(H_t)-C(\rho_t^2+\rho_t a_t),
    \end{array}
\end{equation*}
where $c$ and $C$ are positive constants.

\noindent
Denote by $S_t=\sum_{l=0}^{t-1}(\rho_l^2+\rho_l a_l)$. Note that $(S_t)_{t\in\mathbb{N^*}}$ is a convergent sequence.
So, $$\mathbb{E}[\Lcal_0(H_{t+1})-CS_{t+1}|\Fcal_t] \geq \Lcal_0(H_t)-CS_t.$$

By Doob's theorem we get that
$$\Lcal_0(H_t)-CS_t \longrightarrow \tilde{F} \quad \text{a.s.},$$
where $\tilde{F}$ is a real-valued random variable.

\noindent
And so $\Lcal_0(H_t)\longrightarrow F\in [\min _a q_*(a), \max_a q_*(a)]$ a.s. .

We have that
\begin{equation*}
    \mathbb{E}[\|\nabla\Lcal_0(H_t)\|^2] = \sum_{a \in [k]}\mathbb{E}[\Pi_{H_t}(a)^2(q_*(a)-\Lcal_0(H_t))^2]
\end{equation*}
\begin{equation}
    \Rightarrow \sum_{a \in [k]}\mathbb{E}[\Pi_{H_t}(a)^2(q_*(a)-F)^2]\longrightarrow 0.
\label{eq_sum_F_conv}
\end{equation}

We prove by contradiction that $F\in\{ q_*(1), q_*(2), ..., q_*(k)\}$ a.s..

\noindent
Assume that on $\Omega_0$ (with $\mathbb{P}(\Omega_0)>0$): 

$(q_*(a)-F)^2>0$ a.s., \quad i.e. \quad $\min \{(q_*(a)-F)^2; a\in[k]\}>0$ a.s. in $\Omega_0$.

By \eqref{eq_sum_F_conv} we get via the Cauchy-Schwarz inequality we get that:
\begin{equation*}    \mathbb{E}\left[\mathbbm{1}_{\Omega_0}\frac{1}{k}\left(\sum_{a \in [k]}\Pi_{H_t}(a)\right)^2 \cdot \min \{(q_*(a)-F)^2; a\in[k]\}\right]\longrightarrow 0
\end{equation*}
\begin{equation*}
\Leftrightarrow \mathbb{E}[\underbrace{\mathbbm{1}_{\Omega_0}\min \{(q_*(a)-F)^2; a\in[k]\}}_{\text{strictly positive on $\Omega_0$}}]\longrightarrow 0,    \quad \text{ which is a contradiction.}
\end{equation*}

So, $F\in\{ q_*(1), q_*(2), ..., q_*(k)\}$ a.s.. 
Let $\Omega_l=F^{-1}(q_*(l))$, $l\in[k]$. 
If $\mathbb{P}(\Omega_l)>0$ then
$$\sum_{a\neq l}\mathbb{E}[\Pi_{H_t}(a)^2(q_*(a)-F)^2\mathbbm{1}_{\Omega_l}]\longrightarrow 0.$$
Using that $\sum_{a\neq l}\Pi_{H_t}(a)^2\geq \frac{1}{k-1}(\sum_{a\neq l}\Pi_{H_t}(a))^2$ and that $(q_*(a)-F)^2\geq \min\{|q_*(a)-q_*(l)|; a\neq l\}^2$ we get
\begin{equation*}
    \begin{array}{l}
    \mathbb{E}[(1-\Pi_{H_t}(l))^2\mathbbm{1}_{\Omega_l}]\longrightarrow 0      \\
    \Rightarrow \mathbbm{1}_{\Omega_l}(1-\Pi_{H_t}(l)) \longrightarrow 0 \quad \text{in } L^2  \Rightarrow \mathbbm{1}_{\Omega_l}\Pi_{H_t}(a) \longrightarrow 0 \quad \text{in } L^2, a\neq l.        
    \end{array}
\end{equation*}

So, $\mathbbm{1}_{\Omega_l}\Pi_{H_t}(l) \longrightarrow \mathbbm{1}_{\Omega_l}$ in $L^2$.

And consequently, $\Pi_{H_t}(a)\longrightarrow \mathbbm{1}_{F^{-1}(q_*(a))}$ in $L^2$, $a\in [k]$ which gives the conclusion 
\eqref{eq:convergene_L2_Dirac_omega}.
\end{proof}

\subsection{Proof of Theorem~\ref{thm:cv_nabla_rhocst}}
\begin{proof}
\noindent 
{\bf Proof of assertion~\eqref{eq:limLgammat_rhocst}:} 

Some part of the proof for $\rho_t$ constant is similar to the case when $\rho_t$ is depending on $t$. More specifically,
using exactly the same techniques as in the proof of Proposition~\ref{prop:summability}
we obtain (\ref{eq:inequality_nabla_grad_proof1}) in the form: 
\begin{align}
\rho\mathbb{E}[\|\nabla _H{\cal L}_{\gamma _t}(H_t)\|^2]
+\frac{\gamma _t-\gamma _{t+1}}{2}\mathbb{E}[\| H_{t+1}\|^2]
\leq \mathbb{E}[{\cal L}_{\gamma _{t+1}}(H_{t+1})]-\mathbb{E}[{\cal L}_{\gamma _t}(H_t)]+C_1 \rho^2.    
\label{eq:inequality_nabla_grad_proof1_rhocst}
\end{align}
Note on the other hand that, as before:

- $\mathbb{E}[{\cal L}_{\gamma_{t}}(H_{t})]$ is bounded from above being a difference of a bounded term and a negative part

- Lemma~\ref{lemma:bounded_gH_gt} remains true when $\rho_t$ is constant i.e., the sequence
$\E[ \gamma_t^2 \|H_t\|^2]$ is bounded and the same is  
$\mathbb{E}[\|\nabla _H{\cal L}_{\gamma _t}(H_t)\|^2]$.

Summing up \eqref{eq:inequality_nabla_grad_proof1_rhocst} up to time $t$ we obtain \eqref{eq:limLgammat_rhocst}.
We also obtain that 
\begin{align}
\sum_{s=1}^t    \frac{\gamma_s-\gamma _{s+1}}{2}\mathbb{E}[\| H_{s+1}\|^2] \le c_{1h} \cdot t \cdot  \rho^2, 
\end{align}
for some positive constant $c_{1h}$ and under hypothesis~\eqref{eq:gamma_hyp2} we obtain:
\begin{align}
&
\frac{1}{t}
\sum_{s=1}^t \mathbb{E}[\|\gamma_s H_s\|^2] =  O(\rho),
\\ &  
\frac{1}{t}
\sum_{s=1}^t \mathbb{E}[\|\nabla _H{\cal L}_{0}(H_s)\|^2] = O(\rho). 
\label{eq:summability_L0gtHt_rho_cst}
\end{align}
Note too that the proof of assertions \eqref{eq:summability_L0gtHt}, \eqref{eq:gtHt2_cv}, \eqref{eq:lim_gammatHt} and \eqref{eq:limL0}
can also adapted to $\rho_t=\rho$.

To prove \eqref{eq:lim_regret_zero_rhocst} we use again Lemma 3 from \cite{mei_global_2020}
in the form $\|\nabla_H \Lcal_0(H_t)\| \ge c_0 \cdot \Rcal(\Pi_{H_t})$ and write:
\begin{align}
    O(\rho)= \frac{1}{t}
\sum_{s=1}^t \mathbb{E}[\|\nabla _H{\cal L}_{0}(H_s)\|^2] \ge
c_0^2 \frac{1}{t} \sum_{s=1}^t\Rcal(\Pi_{H_s})^2 \ge
c_0^2 \left( \frac{1}{t}\sum_{s=1}^t\Rcal(\Pi_{H_s}) \right)^2
=c_0^2 \Rcal(\Pibar_t)^2,
\end{align}
where we used that the regret $\Rcal$ is linear in its argument $\Pi$ which proves \eqref{eq:lim_regret_zero_rhocst}.

On the other hand the conclusion \eqref{eq:lim_qstar_rhocst}
follows from
\eqref{eq:lim_regret_zero_rhocst}
using the inequality~\eqref{eq:dem_rhot_regret_dist} for $\Pi = \Pibar_t$.
\end{proof}


\subsection{Implementation of the entropy regularization: formulaes}
\label{sec:formulas_entropy}

We compared L2 and entropy regularization and, for clarity and because these are not available in the literature and on online code repositories, we give below for reference the formulas that were used for the implementation; this is the convention use in all figures of the paper. 
The objective is to maximize
\begin{align}
    \Lcal(t) = \mathbb{E}[R_t]
            - \frac{\gamma_{\text{L2}}(t)}{2} \| H_t \|_2^2
            + \gamma_{\text{Ent}}(t) \, \Omega(\Pi_{H_t}),    
\label{eq:Lcal_entropy}
\end{align}
where $ \Omega(\Pi) = - \sum_{a\in [k]} \Pi(a) \log \Pi(a) $ is the Shannon entropy and $ \gamma_{\text{L2}}(t), \gamma_{\text{Ent}}(t) \ge 0$ are decay schedules.

Sign conventions: 

- we perform gradient ascent on $ \Lcal(t) $ 

- with the L2 penalty term: $ - \frac{\gamma_{\text{L2}}(t)}{2} \| H_t \|^2 $,  whose gradient contribution is $ - \gamma_{\text{L2}}(t) H_t $, 

- and an entropy term: $  \gamma_{\text{Ent}}(t) \, \Omega(\Pi_{H_t}) $
whose gradient contribution is $  \gamma_{\text{Ent}}(t) \nabla_H \Omega(\Pi_{H_t}) $.

Entropy gradient:
\begin{align}
&       \nabla_{H(b)} \Omega(\Pi_{H_t}) =-\sum_a (\log \Pi_{H_t}(a) + 1) \nabla_{H(b)} \Pi_{H_t}(a) 
= -\sum_a \Pi_{H_t}(a) (\log \Pi_{H_t}(a) + 1) (1_{a=b}-\Pi_{H_t}(b)). 
\end{align}

Recall that action $A_t$ is sampled according to $\Pi_{H_t}$ so the unbiased gradient is given by the formula:
$$
g_t = \left[r_t - \bar{r}_t - \gamma_{\text{Ent}}(t) - \gamma_{\text{Ent}}(t) \log(\Pi_{H_t}) \right]\odot (1_{A_t} - \Pi_{H_t})
        - \gamma_{\text{L2}}(t) H_t , 
$$ 
where $ \bar{r}_t $ is the running average reward (baseline). This gives the full preference update $H_{t+1} = H_t + \rho_t g_t$.
In fact, following  
\cite{mei_global_2020}('Update 2' section 4.2.1)
we replace $1+\log(\Pi_{H_t})$ by $\log(\Pi_{H_t})$ above as it does not change the unbiasedness of the gradient.

\section{Further numerical results}

\subsection{Grid search for best linear decay schedules for $\rho_t$}

We give here the analogue result of \Cref{fig:constant_rho_no_reg} but here for linear decay schedules  $\rho_t= \frac{c_1}{1+c_2 \cdot t}$.
To be  closer to the realistic setting we plot the empirical regret i.e. $\max_a q_*(a) - R_t$ instead of exact regret $\Rcal$; note that empirical regret is, up to a constant, available during training while $\Rcal()$ is not; in practice to choose a decay schedule only such outputs are available and not $\Rcal()$. We checked that comparison between the decay schedules is the same when true regret is used.

\begin{figure}[ht]
  \vskip 0.2in
  \begin{center}
    \centerline{
    \includegraphics[width=0.75\linewidth]{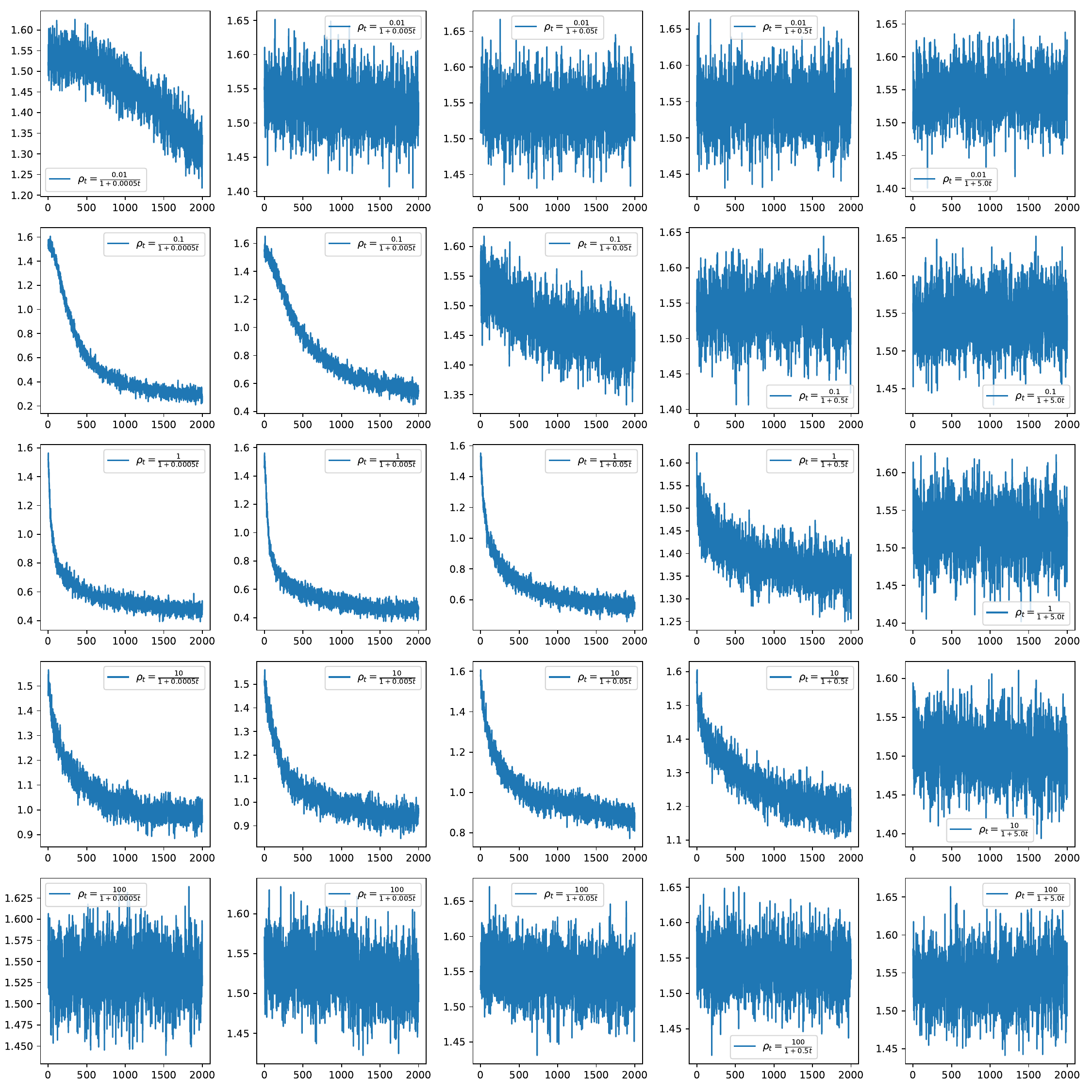}
}

\caption{Extensive grid search for several linear decay schedules of the form $\rho_t= \frac{c_1}{1+c_2 \cdot t}$, both parameters spanning several orders of magnitude. Initial distribution 
$\Pi_{H_0}$ corresponds to 
 $H_0=(5,...,0)$. No regularization (neither L2 nor entropic) is used. The value $\rho_t=\frac{0.1}{1+0.0005\cdot t}$ appears to be the winner. 
 We plot the average empirical regret. This is to be compared with \Cref{fig:constant_rho_no_reg}.
}  \label{fig:grid_search_linear_schedule_rhot}
  \end{center}
\end{figure}

\subsection{Search for best entropy constants}

For the setting in \Cref{sec:baseline}, we give below the empirical results for entropy regularization 
in \Cref{fig:rho01_cst_gamma001to100_entropy};
we set $\rho_t$ with best value obtained in \Cref{fig:constant_rho_no_reg}.
At odds with \Cref{fig:constant_rho_no_reg} we see here that there is no clear winner and the best constant is in the range $0.01$-$1$, with lowest values being best when $t$ is large; accordingly we take for our entropy tests $\gamma_t = \frac{1}{1 + 0.2 t}$ to span all optimal ranges. 

\begin{figure}[ht]
  \vskip 0.2in
  \begin{center}
    \centerline{
    \includegraphics[width=\linewidth]{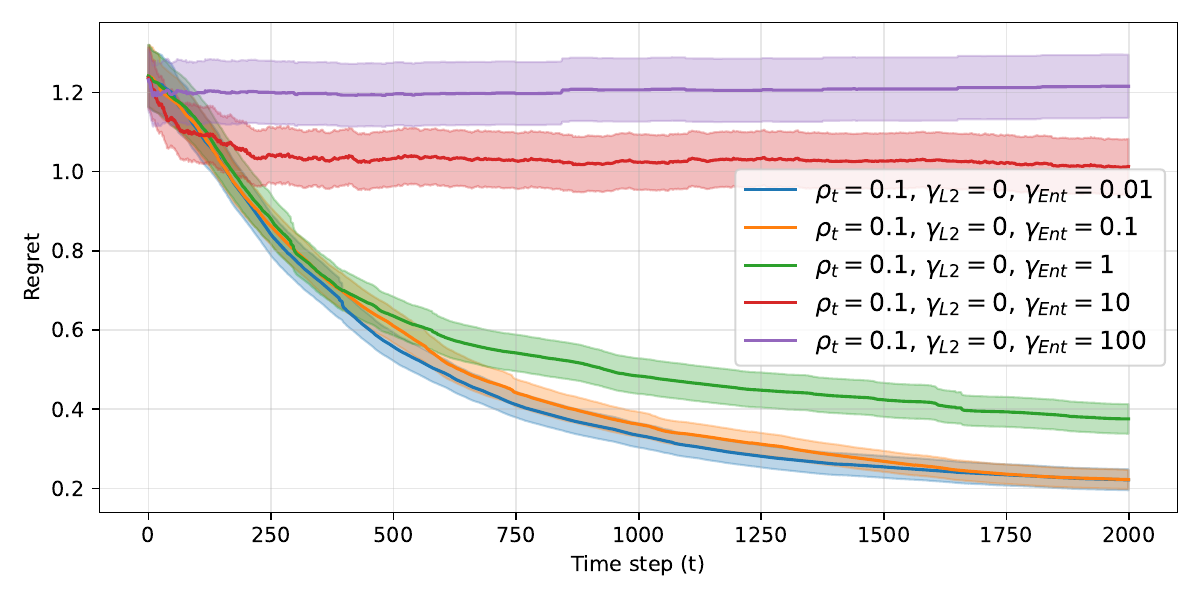}
}
\caption{Search for best $\gamma_t$ constant for entropy.  As before   $H_0=(5,...,0)$. No L2 regularization is used. 
There is no winner and curves seem to overlap quite a bit but 
values in the range 0.01 to 1 seem to belong to the best performing cluster, with small values being increasingly efficient at $t$ gets large. This orients us towards the decay rate  $\gamma_t = \frac{1}{1 + 0.2 t}$ used in latter tests. 
Cf. also  \Cref{fig:constant_rho_no_reg}.
}  \label{fig:rho01_cst_gamma001to100_entropy}
  \end{center}
\end{figure}

\section{Comparison with other classes of algorithms}
\label{sec:comparison_ucb}

One may ask what about other classes of algorithms such as UCB, Thompson Sampling, Adversarial bandit algorithms and so on. The short answer is: we do not compare with other classes of algorithms because our main point here is to investigate the relevance of the L2 regularization within the softmax parameterized policy gradient class, so we are somehow working conditional to the fact that the user has already chosen, for some reason\footnote{There can be many reasons to use softmax MAB over the other procedures, for instance when the number of arms is large, or when the arms averages $q_*$ may drift over time or when the number of arms itself is variable, or when the distributions $R(a)$ of the arms are heavy tailed and so on.}, the softmax policy gradient algorithm and wants, if possible, to improve by regularizing. But, for the curious reader we provide in \Cref{fig:ucb} tests that confirm the common lore for e.g., the UCB algorithm: it performs well when there are not so many arms and the distribution is well behaved; on the contrary the performance deteriorates when the number of arms is increasing and the distribution is heavy tailed. 

\begin{figure}[ht]
    \centerline{
    \includegraphics[width=\linewidth]{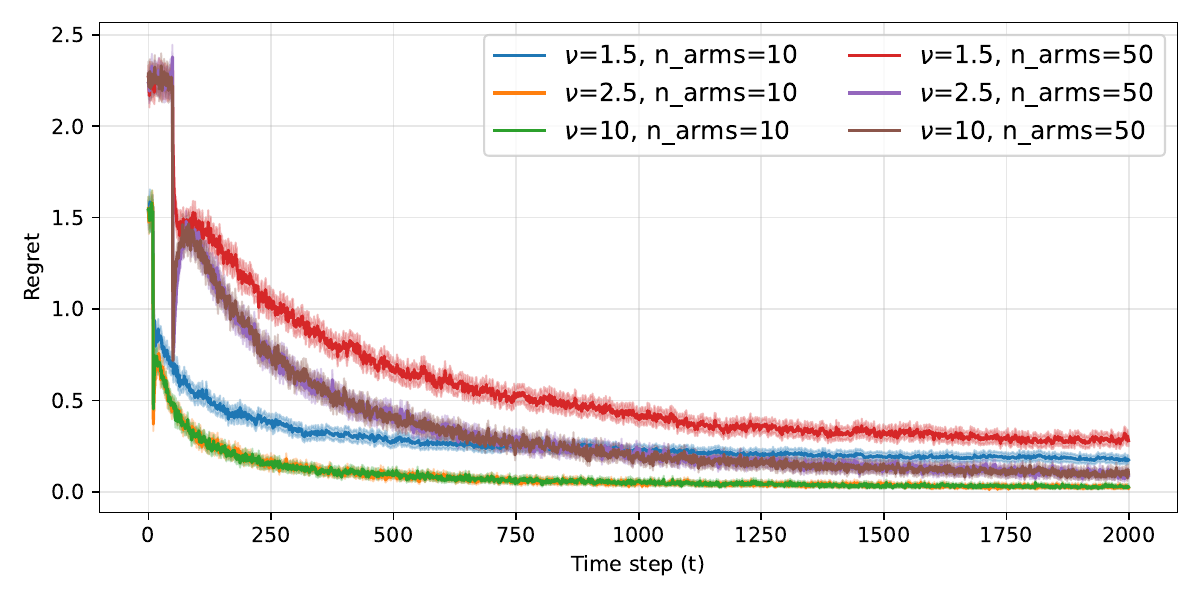}
}
\caption{Performance of the UCB algorithm. The number of degrees of fredom $\nu$ is a proxy for how much the reward distribution is heavy tailed, with $\nu=1.5$ being severely heavy tailed and $\nu=10$ a smooth example. We see that best results are for $\nu=2.5$ and $\nu=10$ and $10$ arms; once we exit this smooth, low arm regime the regret is severely impacted, with results being most sensitive to $\nu$. A quick check with \Cref{fig:cst_rho_compare_L2_reg_entropyk10t15} confirms that L2 regularized MAB outperforms UCB in this case.}
\label{fig:ucb}
\end{figure}

\end{document}